\documentclass[11pt,twoside]{article}

\usepackage{fullpage}

\usepackage{comment}
\usepackage{epsf}
\usepackage{graphicx}
\usepackage{subfigure}
\usepackage{bbm}
\usepackage{color}
\usepackage{nicefrac}
\usepackage{mathtools}
\usepackage{amsfonts}
\usepackage{amsthm}
\usepackage{amsmath, bm}
\usepackage{amssymb}
\usepackage{textcomp}
\usepackage{xcolor}
\usepackage{algorithm}
\usepackage{algorithmic}

\usepackage[utf8]{inputenc}
\usepackage{pgfplots}
\DeclareUnicodeCharacter{2212}{−}
\usepgfplotslibrary{groupplots,dateplot}
\usetikzlibrary{patterns,shapes.arrows}
\pgfplotsset{compat=newest}

\newenvironment{fminipage}%
  {\begin{Sbox}\begin{minipage}}%
  {\end{minipage}\end{Sbox}\fbox{\TheSbox}}

\newtheorem{theorem}{Theorem}

\newtheorem{proposition}{Proposition}
\newtheorem{lemma}{Lemma}

\newtheorem{remark}{Remark}

\usepackage{natbib}
\setcitestyle{round}

\newlength{\widebarargwidth}
\newlength{\widebarargheight}
\newlength{\widebarargdepth}

\makeatletter
\long\def\@makecaption#1#2{
       \vskip 0.8ex
       \setbox\@tempboxa\hbox{\small {\bf #1:} #2}
       \parindent 1.5em 
       \dimen0=\hsize
       \advance\dimen0 by -3em
       \ifdim \wd\@tempboxa >\dimen0
               \hbox to \hsize{
                       \parindent 0em
                       \hfil 
                       \parbox{\dimen0}{\def\baselinestretch{0.96}\small
                               {\bf #1.} #2
                               } 
                       \hfil}
       \else \hbox to \hsize{\hfil \box\@tempboxa \hfil}
       \fi
       }
\makeatother

\long\def\comment#1{}

\newcommand\indpt{\protect\mathpalette{\protect\independenT}{\perp}} 
\def\independenT#1#2{\mathrel{\rlap{$#1#2$}\mkern2mu{#1#2}}}
\DeclareMathOperator{\Var}{var}

\newcommand{\EE}{\mathop{{}\mathbb{E}}}

\newcommand{\ind}[1]{\ensuremath{{\mathbb{I}\left\{ #1 \right\}}}}

\newcommand{\1}{\ensuremath{{\sf (i)}}}
\newcommand{\2}{\ensuremath{{\sf (ii)}}}

\newcommand{\Xspace}{\ensuremath{\mathcal{X}}}

\newcommand{\order}{\ensuremath{\mathcal{O}}}
\newcommand{\ordertil}{\ensuremath{\widetilde{\order}}}

\newcommand{\tmix}{\ensuremath{\mathsf{t_{mix}}}}

\newcommand{\Tmix}{\ensuremath{\mathsf{T_{mix}}}}

\newcommand{\Pmat}{\ensuremath{\boldsymbol{P}}}

\newcommand{\some}{\ensuremath{\textsc{WingIt}}}

\newcommand{\Mhat}{\ensuremath{\widehat{M}}}

\newcommand{\MSE}{\ensuremath{\mathsf{MSE}}}

\newcommand{\dtv}{\ensuremath{\mathsf{d_{TV}}}}

\newcommand{\reals}{\ensuremath{\mathbb{R}}}

\newcommand{\deltabar}{\ensuremath{\overline{\delta}}}

\newcommand{\Dset}{\ensuremath{\mathcal{B}}}

\newcommand{\Iset}{\ensuremath{\mathcal{H}}}

\newcommand{\Dsetone}{\ensuremath{\mathcal{D}}}
\newcommand{\Isetone}{\ensuremath{\mathcal{I}}}

\newcommand{\T}{\ensuremath{\tau}}

\begin{document}

\begin{center}

  {\bf{\LARGE{\mbox{Just Wing It: Near-Optimal Estimation of Missing} \\ 
  Mass in a Markovian Sequence}}}

\vspace*{.2in}

{\large{
\begin{tabular}{ccc}
 Ashwin Pananjady$^{\star, \dagger}$, Vidya Muthukumar$^{\dagger, \star}$, Andrew Thangaraj$^{\ddagger}$
\end{tabular}
 }}
 \vspace*{.2in}

 \begin{tabular}{c}
 Schools of Industrial and Systems Engineering$^\star$ and
 Electrical and Computer Engineering$^\dagger$, \\
 Georgia Institute of Technology \\
 Department of Electrical Engineering$^\ddagger$, IIT Madras
 \end{tabular}

\vspace*{.2in}

Original: April 8, 2024; \;\;\;Revised: October 5, 2024

\vspace*{.2in}

\end{center}

\begin{abstract}
We study the problem of estimating the stationary mass---also called the unigram mass---that is missing from a single trajectory of a discrete-time, ergodic Markov chain. This problem has several applications---for example, estimating the stationary missing mass is critical for accurately smoothing probability estimates in sequence models. While the classical Good--Turing estimator from the 1950s has appealing properties for i.i.d. data, it is known to be biased in the Markovian setting, and other heuristic estimators do not come equipped with guarantees. Operating in the general setting in which the size of the state space may be much larger than the length $n$ of the trajectory, we develop a linear-runtime estimator called \emph{Windowed Good--Turing} (\textsc{WingIt}) and show that its risk decays as $\ordertil(\Tmix/n)$, where $\Tmix$ denotes the mixing time of the chain in total variation distance. Notably, this rate is independent of the size of the state space and minimax-optimal up to a logarithmic factor in $n / \Tmix$. We also present an upper bound on the variance of the missing mass random variable, which may be of independent interest.
We extend our estimator to approximate the stationary mass placed on elements occurring with small frequency in the trajectory.  Finally, we demonstrate the efficacy of our estimators both in simulations on canonical chains and on sequences constructed from natural language text.
\end{abstract}

\section{Introduction}

Two classical problems in statistical analysis—relevant to both design of experiments and inference—are those of assessing \emph{sample coverage} and \emph{discovery probability}. Given a ``training" sequence $X^n = (X_1, X_2, \ldots, X_n)$ of random examples in some unknown sample space, the latter question concerns the probability with which an independent “test” observation $Y$ will be a \emph{discovery}, in that it is an element of the sample space that was unseen at training time. Equivalently, we are interested in estimating the \emph{missing mass} in the training sample $X^n$, i.e.~$\Pr\{Y \notin \{X_1,\ldots,X_n\}\}$.

This problem has roots in statistical analysis for ecology~\citep{fisher1943relation}, and also has important applications across genomics~\citep{favaro2012new} as well as speech and language modeling~\citep{church1991probability,chen1999empirical}. Let us give a few operational examples. For a first example from genomics~\citep{lijoi2007bayesian}, suppose we have performed genome sequencing on several genes of an organism as part of training data, and we are now interested in whether there is value in performing additional sequencing. Then the missing mass exactly measures the probability that we discover a new gene with additional sequencing, and an accurate estimate of this quantity can guide decisions about whether or not to sequence further.
For a second example, consider the problem of building a probability model for a language corpus~\citep{ney1994structuring}. Many heuristic ``smoothing” estimators have been developed for estimating these probability models~\citep[e.g.,][]{ney1994structuring,jelinek1985probability,gale1995good}. A crucial component of these smoothing techniques is an estimate of the missing mass, since one would like to account for the (non-trivial) possibility that a word exists in the population corpus but has not yet been observed in the training data.
Besides these  examples, estimates of the missing mass are also used in so-called competitive distribution estimation~\citep{orlitsky2015competitive} and in estimating other functionals of distributions such as their entropy~\citep{vu2007coverage}. Recent connections have also been made between the missing mass in a fact sequence and the propensity of large language models to ``hallucinate" spurious facts~\citep{kalai2024calibrated}.

Many estimators with provable---and in fact minimax-optimal---guarantees exist for the case where the training data are exchangeable~\citep{good1953population,lijoi2007bayesian,mcallester2000convergence}. While the exchangeability assumption is reasonable in some applications, for example ecology~\citep{shen2003predicting,colwell2012models}, it is clearly limiting in both genomics and speech or language applications, where temporal dependencies exist between the examples. The simplest form of such temporal dependence is Markovian structure, and, as articulated repeatedly in the literature~\citep{hao2018learning,chandra2021good,skorski2020missing}, handling such structure in a principled fashion is an important first step for estimation of missing mass in temporally dependent training sequences. In spite of a significant body of work motivated by this topic, there still do not exist consistent estimators for missing mass functionals in general classes of Markovian sequences.

In this paper, we propose and theoretically analyze an estimator for the missing mass in a Markovian data sample, and variants for related problems.
To make things concrete, suppose our stochastic process $X^n := (X_1, \ldots, X_n)$ is modeled by a stationary Markov chain $(\Pmat, \pi)$ on a finite but \emph{unknown} state space $\Xspace$. We will make no assumptions on the alphabet size $|\Xspace|$, and will be interested in also capturing the practically relevant large-alphabet setting, i.e. where $|\Xspace| \gg n$.
Here $\pi = (\pi_x)_{x \in \Xspace}$ denotes the unique stationary distribution of the chain, and the matrix $\Pmat \in [0,1]^{|\Xspace| \times | \Xspace|}$ denotes the transition probability matrix of the Markov chain. 
We assume for convenience that $X_1 \sim \pi$, but this assumption can be straightforwardly relaxed\footnote{As is standard in the literature, one can handle the case of arbitrary $X_1$ by letting the chain burn in for a certain number of steps until the new ``initial" distribution becomes close to the stationary measure.}. 

As previously mentioned, our primary goal is to estimate the mass of the Markov chain that is missing from the random sample $X^n$. Motivated by the questions above, we focus on the \emph{stationary} missing mass of the chain, given by
\begin{align}
M_{\pi}(X^n) := \sum_{x \in \Xspace} \pi_x \cdot\ind{ x \notin \{X_1, \ldots, X_n\} },
\end{align}
where $\pi_x$ is the probability assigned by the stationary distribution $\pi$ to element $x \in \Xspace$.
Note that $M_{\pi}(X^n)$ is a \emph{random functional}, as it depends not only on the parameters of the chain but also the random sample $X^n$.
An equivalent definition---which resembles the description above---is given by
\begin{align} \label{eq:M-Y}
M_{\pi}(X^n) = \EE_{ \substack{Y \sim \pi \\ Y \indpt X^n}} \left[ \ind{ Y \notin \{X_1, \ldots, X_n\} } \right],
\end{align}
where $U\indpt V$ denotes that random variables $U$ and $V$ are independent.

The missing mass is not the only functional that is relevant to discovery probabilities. A closely related functional is the \emph{small-count} stationary probability, which measures the probability of seeing an element that had a frequency at most $\zeta$ in the training sequence~\citep{lijoi2007bayesian,favaro2012new}. In particular, consider the estimand
\begin{align} \label{eq:small-count}
M_{\pi, \leq \zeta}(X^n) = \EE_{ \substack{Y \sim \pi \\ Y \indpt X^n}} \left[ \ind{ Y \text{ appears at most } \zeta \text{ times in }\{X_1, \ldots, X_n\} } \right].
\end{align}
We will present detailed results for estimating the functional $M_{\pi, \leq \zeta}$ in Section~\ref{sec:extension}, focusing up until that point on the missing mass.

Our goal is to produce an estimator $\Mhat: \mathcal{X}^n \to [0, 1]$ with minimum risk, where risk is measured using the mean squared error. In particular, for an estimand $M: \mathcal{X}^n \to [0, 1]$ and estimator $\Mhat$, we write
\begin{align} \label{eq:MSE}
\MSE(\Mhat, M) = \EE_{X^n} \left[ | \Mhat(X^n) - M(X^n)|^2 \right].
\end{align}
Above, the expectation is taken over any other sources of randomness in $\Mhat$ in addition to the randomness in the sequence $X^n$.

To set up some additional notation, we let $\| \mu - \nu \|_{\mathsf{TV}}$ denote the total variation distance between two probability measures $\mu$ and $\nu$ defined on the same space.
Throughout, we assume that the Markov chain is  ergodic and mixes in finite time. In particular,
let $\tmix(\epsilon)$ denote the mixing time of the chain to within total variation $\epsilon \in (0, 1/2]$ of the stationary measure, i.e.
\begin{align} \label{eq:tmix}
\tmix(\epsilon) := \min \; \left\{ t \in \mathbb{N}: \max_{x \in \Xspace} \| e_x^\top \Pmat^t - \pi^\top \|_{\mathsf{TV}} \leq \epsilon \right\},
\end{align}
where $e_x$ denotes the indicator vector on element $x \in \Xspace$ and $\pi$ is viewed as a $|\Xspace|$-dimensional column vector.
The quantity $\tmix(1/4)$ is typically called the mixing time of the chain, and so we will write $\Tmix := \tmix(1/4)$.
It is straightforward to show (see, e.g.~\cite{levin2017markov}) that
\begin{align} \label{eq:tmix-logs}
    \tmix(\epsilon) \leq \Tmix \cdot \log(1/\epsilon) \text{ for all } \epsilon < 1/4.
\end{align}

\subsection{Related work}

The problem of estimating missing mass of a random sequence, where each element is drawn from an arbitrarily large sample space, was studied as far back as the 1800s by~\citet{laplace1814philosophical}, who proposed the first among the class of ``add-constant" estimators. These estimators have seen a line of theoretical and empirical follow-up work~\citep{krichevsky1981performance,gale1994wrong}, with special attention being paid to the add-$1/2$-estimator~\citep{krichevsky1981performance}. Instead of outputting the normalized empirical frequencies of elements as a maximum-likelihood estimator would, these estimators add a constant to the (un-normalized) empirical frequency prior to normalization. In the process, they output a non-zero missing mass probability.

A notable and groundbreaking result of~\citet{good1953population}---attributed also to Turing---moved away from the class of add-constant estimators and proposed to estimate the missing mass via the normalized frequency of elements appearing \emph{once} in the sequence. In particular, letting $\phi_s(X^n)$ denote the number of distinct elements of $\Xspace$ that have appeared $s$ times in the sample $X^n$, the celebrated Good--Turing estimator for the missing mass is given by
\begin{align} \label{eq:GT-orig}
\Mhat_{\mathsf{GT}} = \frac{\phi_1(X^n)}{n}.
\end{align}
The estimator has been applied to diverse areas~\cite[see, e.g.,][]{song1999general,gale1992method,church1991probability} and has also seen intense theoretical study in the last three decades~\citep[see, e.g.,][]{mcallester2000convergence,drukh2005concentration,orlitsky2003always}. In particular, several analyses of fine-grained properties of the estimator now exist for the i.i.d. setting~\citep[see, e.g.,][]{chandra2019improved,rajaraman2017minimax,acharya2018improved}, and variants of the estimator have also been proposed and studied~\citep{gandolfi2004nonparametric,favaro2016rediscovery,painsky2022convergence,painsky2023generalized}. While most analyses focus on additive error---e.g., the mean squared error of estimating the missing mass $M_{\pi}(X^n)$---the multiplicative error metric has also been studied~\citep{ohannessian2012rare,mossel2019impossibility,ayed2021consistent,hamou2017concentration,grabchak2017asymptotic}.
Besides estimation, the missing mass random variable $M_{\pi}(X^n)$ has itself generated a lot of interest in the i.i.d. setting---its concentration properties have been thoroughly studied, and several analysis techniques have been developed along the way~\citep{mcallester2003concentration,berend2012missing,berend2013concentration}.

In contrast to the i.i.d. setting, the Markovian setting has received relatively sparse treatment, in spite of being the main setting---smoothing in language models---that motivated some of the initial papers on the theory of the subject~\citep{mcallester2000convergence,mcallester2001learning}. Some such papers include results for sticky Markov chains~\citep{chandra2022missing} and rank-$2$ Markov chains~\citep{chandra2021good}. These papers mainly study the performance of the Good--Turing estimator and/or certain scaled variants of it, and also give sufficient conditions under which Good--Turing can succeed for Markovian data. However, these conditions are restrictive, and it is not yet known if one can perform consistent, let alone minimax-optimal estimation of the missing mass in the general Markovian setting. Concentration of missing mass in the Markovian setting has also received some recent interest~\citep{skorski2020missing}---we will discuss this paper in greater detail in the sequel.

The problem of estimating the small-count probability~\eqref{eq:small-count} has also been studied in the literature~\citep{lijoi2007bayesian}, along with the related problem of estimating the \emph{exact count probability}, i.e., the probability of elements occurring \textit{exactly} $\zeta$ times~\citep{good1953population}. Estimators of the count probability have been developed for i.i.d. samples, and several theoretical results are also available for this setting~\citep{mcallester2001learning,drukh2005concentration,acharya2013optimal}. The Markovian case, however, does not seem to have been theoretically studied.

Finally, we mention that besides the missing mass and count probabilities, other estimation and prediction problems~\citep{hao2018learning,han2021optimal,wolfer2019minimax} and bounds on ``surprise" probabilities~\citep{norris2017surprise} have been studied for Markov chains. It is worth noting that the size of the state space appears explicitly as a parameter in these results.

\subsection{Contributions and organization}

Our contributions are summarized below:
\begin{itemize}
\item In Section~\ref{sec:wingit}, we propose an estimator for stationary missing mass in the Markov setting called the \emph{Windowed Good--Turing}, or \textsc{WingIt}, estimator. Our estimator is based on the viewpoint of the Good--Turing estimator as a leave-one-out estimator, a perspective that we review (for the i.i.d. setting) and develop in Section~\ref{sec:iid}.
\item In Theorem~\ref{thm:main-upper-bound} of Section~\ref{sec:theory}, we provide a risk bound on the \textsc{WingIt} estimator, showing that it attains mean squared error on the order $\Tmix/n$ up to a logarithmic factor in $n/\Tmix$. This matches, up to this logarithmic factor, the minimax lower bound for missing mass estimation in mixing Markov chains~\citep{chandra2022missing}.
\item Aside from providing an estimator for the missing mass, we also analyze the missing mass functional $M_{\pi}(X^n)$ as a random variable, and show in Theorem~\ref{thm:variance-bound-MM} that its variance is bounded on the order $\Tmix^2/n$ up to a logarithmic factor. This bound follows from a stability property of the \textsc{WingIt} estimator and constitutes, to our knowledge, the first variance bound on the missing mass in the Markovian setting, complementing a one-sided bound due to~\cite{skorski2020missing}.
\item In Section~\ref{sec:extension}, we present an extension of our methodology for estimating the small-count probability~\eqref{eq:small-count}. 
Note that this generalizes the problem of estimating the stationary missing mass, which corresponds to the case $\zeta = 0$. This result, stated as Theorem~\ref{thm:main-upper-bound-zeta}, appears to improve analogous guarantees from the literature even for i.i.d. samples.
\item In Section~\ref{sec:expts}, we provide simulations on some synthetic Markov chains and on natural language text. These experiments corroborate our theory while showing how the \textsc{WingIt} estimator can significantly outperform the vanilla Good--Turing estimator. We also (empirically) explore an automatic and data-dependent tuning method for the window size hyperparameter in our \textsc{WingIt} estimator.
\end{itemize}

\paragraph{Notation:}
For two real numbers $a$ and $b$, let $a \wedge b = \min\{a, b\}$ and $a \vee b = \max\{a, b\}$.
Let $[n]$ denote the set of natural numbers less than or equal to $n$. For an index set $P \subseteq [n]$, let $\bm{X}_P = \{ X_i \}_{i \in P}$ denote the set of random variables corresponding to that index set. The set $\bm{X}_{[n]}$ thus contains all random variables in the sequence $X^n$. With a slight abuse of notation, we let $X_P = (X_i)_{i \in P}$ denote the sequence of random variables with indices in $P$, ordered canonically. 
For two sequences indexed by $u$, we use the notation $f(u) \lesssim g(u)$ to mean that there exists some absolute positive constant $C$ that is independent of all problem parameters, such that $f(u) \le C \cdot g(u)$ for all $u$. We use the notation $f(u) \gtrsim g(u)$ when $g(u) \lesssim f(u)$. We write $f(u) \asymp g(u)$ if both relations $f(u) \gtrsim g(u)$ and $g(u) \lesssim f(u)$ hold. Logarithms are taken to the base $e$. We use $(c, C)$ to denote universal positive constants that could be different in each instantiation.

\section{From i.i.d. to Markov: Revisiting Good--Turing} \label{sec:iid}

First, let us revisit the special case where the samples $X^n$ are i.i.d. and the stationary distribution $\pi$ corresponds to the probability mass function from which each sample is drawn.
Here, the celebrated Good--Turing estimator~\eqref{eq:GT-orig} of $M_{\pi}(X^n)$ is given by the number of symbols that appear \emph{once} in $X^n$ divided by the sample size $n$. As articulated in the literature~\citep[e.g.][]{mcallester2001learning}, this quantity can be thought of as a \emph{leave-one-out} estimate of the functional that repeatedly simulates a placeholder for $Y$ from the given sample and approximately evaluates the indicator in Eq.~\eqref{eq:M-Y}. In particular, consider the collection of estimators given by the random variables
\begin{align} \label{eq:LOO-GT}
\Mhat^{(i)} = \ind{ X_i \notin \{ X_1, \ldots, X_{i-1}, X_{i+1}, \ldots, X_n \} } \quad \text{ for } \quad i = 1,\ldots,n.
\end{align}
We can then equivalently write the Good--Turing estimator~\eqref{eq:GT-orig} as
\begin{align} \label{eq:GT-equiv}
    \Mhat_{\mathsf{GT}} =  \frac{1}{n} \sum_{i=1}^n \Mhat^{(i)}.
\end{align}
Clearly, the random variable $X_i$ ``simulates" drawing a fresh sample $Y$ from $\pi$ independently of the subsequence $(X_1, \ldots, X_{i-1}, X_{i+1},\ldots, X_n)$, and inspecting Eq.~\eqref{eq:M-Y} yields that $\Mhat^{(i)}$ is an unbiased estimator of $M_{\pi}(X_1,\ldots,X_{i-1},X_{i+1},\ldots,X_n)$. Using the i.i.d. nature of the observations, one can then show that~\citep{mcallester2000convergence}
\[
|\EE[M_{\pi}((X_1,\ldots,X_{i-1},X_{i+1},\ldots,X_n))] - M_{\pi}(X^n)]| \lesssim n^{-1},
\]
so that we have a near-unbiased estimator of the quantity $\EE[M_{\pi}(X^n)]$. Coupling this observation with additional arguments that bound the variance of the estimator $\Mhat_{\mathsf{GT}}$ and estimand $M_{\pi}$~\citep{mcallester2000convergence,mcallester2003concentration}, we obtain a bound on the mean-squared error of the Good--Turing estimator.

It is instructive to re-examine the central pitfall of the Good--Turing estimator for a Markov chain, which is that strong local dependencies between adjacent samples in the Markov chain induce non-vanishing bias.
This argument has been sketched before (see, e.g.~\cite{chandra2022missing}), but we nevertheless give a brief, self-contained illustrative example and a heuristic calculation of the bias below.
Let $\Xspace = [k]$ for some $k \gg n$ and consider transition kernels $\Pmat \in [0, 1]^{k \times k}$ of the form 
\begin{align} \label{eq:sticky-chain}
\Pmat = (1 - p) \bm{I} + p \bm{1} \pi^\top, 
\end{align}
where $\bm{I}$ and $\bm{1}$ denote the identity matrix and all-1s column vector of suitable dimensions. Such a transition kernel gives rise to a so-called ``sticky", or lazy, Markov chain having stationary distribution $\pi$ and mixing time 
$\frac{1}{2p} \leq \Tmix \leq \frac{2}{p}$ for all $k \geq 2$ and $p \in (0, 1/2]$ (see~Lemma~\ref{lem:stickymixingtime}). Thus, as the probability $p$ becomes small, the mixing time becomes proportionally large. 

Now suppose that $\pi = \frac{1}{k} \bm{1}$, so that the stationary distribution is the uniform distribution on $k$ elements.
Due to the stickiness of the chain---i.e., its propensity to remain in its current state---we will see $K \asymp np $ unique elements in a typical sample $(X_1,\ldots,X_n)$ of the chain.
The stationary missing mass of the chain will hence be given, in expectation, by
\[
\EE [M_{\pi}(X^n)] = \frac{k - \EE[K]}{k} \geq \frac{k - Cnp}{k}
\]
for some universal constant $C > 0$. In particular, if $k\ge Cn$ and $p \le 1/4$, we have $\EE [M_{\pi}(X^n)] \geq 3/4$.

On the other hand, the Good--Turing estimator will obey $\EE [\Mhat_{\mathsf{GT}}(X^n) ] \leq p$. This is because for all $i \in [n]$, we have
\begin{align}\label{eq:gtbias_local}
\EE [\Mhat_{\mathsf{GT}}(X^n) ] = \EE[\Mhat^{(i)}] = \Pr\{ X_i \notin \{ X_1, \ldots, X_{i - 1}, X_{i + 1}, \ldots, X_n \} \leq \Pr\{ X_i \neq X_{i-1} \} \leq p.
\end{align}
Consequently, for $k \ge Cn$ and $p \le 1/4$, we have
\begin{align} \label{eq:bias-sticky}
\EE[|M_{\pi}(X^n) - \Mhat_{\mathsf{GT}}(X^n)|] \overset{\1}{\ge} |\EE[M_{\pi}(X^n) - \Mhat_{\mathsf{GT}}(X^n)]| \ge \frac{3}{4} - p \geq \frac{1}{2},
\end{align}
where step $\1$ follows by Jensen's inequality.
In words, the Good--Turing estimator has constant, non-vanishing bias for sticky Markov chains in which the stationary distribution is uniform on a large state space. 
%Letting $Z=|M_{\pi}(X^n) - \Mhat_{\mathsf{GT}}(X^n)|$, 
In addition, we have
\small
\begin{align*}
\Pr\left\{ |M_{\pi}(X^n) - \Mhat_{\mathsf{GT}}(X^n)| \geq 1/4 \right\} &\overset{\1}{\geq} \Pr\left\{ |M_{\pi}(X^n) - \Mhat_{\mathsf{GT}}(X^n)| \geq \frac{1}{2} \cdot \EE[|M_{\pi}(X^n) - \Mhat_{\mathsf{GT}}(X^n)|] \right\} \\
&\overset{\2}{\geq} \frac{1}{4} \cdot \frac{1}{4} = \frac{1}{16},
\end{align*}
\normalsize
where step $\1$ follows by Eq.~\eqref{eq:bias-sticky}, and step $\2$ follows from the Paley--Zygmund inequality \mbox{$\Pr(Z\ge\theta \EE[Z]) \ge (1-\theta)^2\EE[Z]^2$} for any random variable $Z\in[0,1]$, applied with $\theta=1/2$. Thus, the Good--Turing estimator is \emph{inconsistent}, in that its error $|M_{\pi}(X^n) - \Mhat_{\mathsf{GT}}(X^n)|$ cannot converge in probability to zero even as $n \to \infty$.
This phenomenon is also empirically illustrated in Figure~\ref{fig:stick0.5} in Section~\ref{sec:expts}.

While the inconsistency of the Good--Turing estimator is unfortunate, we next build on some important design principles sketched here in order to develop a consistent estimator.

\section{Methodology: The Windowed Good--Turing estimator} \label{sec:wingit}

In this section, we describe a natural modification of the Good--Turing estimator, interpreted through the leave-one-out lens~\eqref{eq:LOO-GT}, that mitigates the above issues and estimates the stationary missing mass at a minimax-optimal rate.

\subsection{A first step: Modifying the ``leave-one-out" estimator to reduce its bias}

The central issue with the original leave-one-out estimator $\Mhat^{(i)}$ (Eq.~\eqref{eq:LOO-GT}) when applied to Markov chains lay in the strong dependencies induced between adjacent samples of the chain: As demonstrated by Eq.~\eqref{eq:gtbias_local}, successive samples $X_i$ and $X_{i-1}$ are tightly coupled through the structure of the transition kernel and very far from being independent.
To mitigate this issue, we modify the leave-one-out estimator to a ``leave-a-window-out" estimator that removes the samples that are adjacent to $X_i$ before computing the corresponding estimator.
We first illustrate the idea for $i = n$ for simplicity.
Recall that $\Mhat^{(n)} = \ind{X_n \notin \{X_1,\ldots,X_{n-1}\}}$.
Instead of using the leave-one-out subsequence $(X_1,\ldots,X_{n-1})$ as a proxy for $X^n$, we use the slightly smaller subsequence $(X_1, \ldots, X_{n - \T})$. As long as we choose some fixed $\T \gg \Tmix$, we should expect that $(X_1, \ldots, X_{n - \T})$ is nearly independent of $X_n$.
Accordingly, we define the estimator
\begin{align} \label{eq:Mn}
\Mhat^{(n)}_{\T} := \ind{ X_n \notin \{X_1, \ldots, X_{n - \T} \} }.
\end{align}
To develop some heuristic intuition, suppose for the moment that $X_n$ were exactly independent of $(X_1, \ldots, X_{n - \T})$, and recall that the marginal distribution of $X_n$ is the stationary measure $\pi$. Then by construction, the random variable $\Mhat^{(n)}_{\T}$ would be an unbiased estimator of $M_{\pi}(X^{n - \T})$, which one should expect, in turn, to be close to the desired missing mass $M_{\pi}(X^n)$ provided $\T$ is not too large. In other words, the estimate $\Mhat^{(n)}_{\T}$ will have a small bias that can be controlled via a suitable choice of window size $\T$. 

\subsection{\textsc{WingIt}: Averaging an ensemble of windowed leave-one-out estimators}

Above, we have sketched how to produce a single random variable $\Mhat^{(n)}_{\T}$ with small estimation bias. However, we still have the issue of variance. How do we construct multiple estimators like $\Mhat^{(n)}_{\T}$ and average over them as in Eq.~\eqref{eq:GT-equiv}?
The natural idea to construct the $i$-th such estimator is to inspect the definition~\eqref{eq:Mn}, replace $X_n$ with $X_i$, and the set $\{X_1, \ldots, X_{n - \T} \}$ with the portion of the sequence $X^n$ that should behave as nearly independent of $X_i$.
Concretely, for each $i \in [n]$, define the index sets 
\begin{align} \label{eq:index-sets}
\Dsetone_i = \{k \in [n]: |k - i| < \T\} \quad \text{ and } \quad \Isetone_i = [n] \setminus \Dsetone_i.
\end{align}
In words the set $\Dsetone_i$ contains indices that are close to $i$, so that if $\T \gg \tmix$ then we should expect $\bm{X}_{\Dsetone_i}$ to be the set of random variables in the sequence that depend significantly on $X_i$. On the other hand, the complementary set $\Isetone_i$ is the index set of random variables that are nearly independent of $X_i$. 
The above intuition then leads naturally to the estimator
\begin{align} \label{eq:Mi}
\Mhat^{(i)}_{\T} := \ind{ X_i \notin \bm{X}_{\Isetone_i}  },
\end{align}
which generalizes $\Mhat^{(n)}_{\T}$ to any index $i$. Finally, we combine these estimates to reduce variance, creating
the final estimator
\begin{align}
\Mhat_{\some}(\T) = \frac{1}{n} \sum_{i = 1}^{n} \Mhat^{(i)}_{\T}. \label{eq:estimator-some}
\end{align}
Note that if $\T = 1$, then we recover the Good--Turing estimator~\eqref{eq:GT-equiv}. See Figure~\ref{fig:blocks} later in the paper for an illustration of the windowing procedure.

\subsection{Linear time implementation of \textsc{WingIt}}

As presented in Eqs.~\eqref{eq:Mi} and~\eqref{eq:estimator-some}, the WingIt estimator can be naively implemented in $\order(n^2)$ time, since each estimator $\Mhat_\T^{(i)}$ can be computed by searching through the entire sequence. In this section, we show that, in fact, the entire estimator $\Mhat_{\some}(\T)$ can be computed in time $\order(n)$ time for any value of the window size $\T$. The computational complexity of the \textsc{WingIt} estimator is thus comparable to that of the vanilla Good--Turing estimator~\eqref{eq:GT-orig}.
Given a sequence $X^n$ and a natural number $\T$, the \textsc{WingIt} estimator proceeds via two passes through the data as shown in Algorithm~\ref{alg:wingit}.
\begin{algorithm}
\caption{Linear time implementation of the \textsc{WingIt} estimator.}\label{alg:wingit}
\begin{algorithmic}[1]
\REQUIRE Sequence $X^n=X_1,\ldots,X_n$, Natural number $\T$
\STATE Initialize \texttt{locations} as a dictionary
\FOR{$i=1,\ldots,n$}
    \IF{$X_i\notin$ \texttt{locations}}
        \STATE Initialize \texttt{locations}$[X_i]$ as a list
    \ENDIF
    \STATE Append $i$ to \texttt{locations}$[X_i]$
\ENDFOR
\STATE $\Mhat \gets 0$
\FOR{$i=1,\ldots,n$} 
    \IF{ \texttt{locations}$[X_i]$.first $>i-\T$ \textbf{and} \texttt{locations}$[X_i]$.last $<i+\T$}
        \STATE $\Mhat \gets \Mhat + 1/n$ 
    \ENDIF
\ENDFOR
\RETURN $\Mhat$
\end{algorithmic}
\end{algorithm}

\vspace{-2mm}
\paragraph{Correctness:} It suffices to show that the condition in Step 10 evaluates to \texttt{True} only when $\ind{ X_{i} \notin \bm{X}_{\Isetone_{i}}}=1$.
By construction, the list \texttt{locations}[$X_{i}$] contains indices $k \in [n]$ sorted in increasing order such that $X_k = X_{i}$. 
So, if the first (i.e. smallest) and last (i.e., largest) element of the list \texttt{locations}$[X_{i}]$ are within $(i-\T,i+\T)$, then we have that $X_{i} \notin \bm{X}_{\Isetone_{i}}$.

\vspace{-2mm}
\paragraph{Running time:} The for loop in Steps~2--7 requires a single pass through the data. In the for loop in Steps~9--12, for each value of $i$, we access only the first and last element of the list \texttt{locations}$[X_{i}]$, which takes two operations in a Python implementation. The total running time of both loops is therefore $\order(n)$, resulting in an overall algorithm that runs in linear time.

\vspace{-2mm}
\paragraph{Memory:} The memory requirement is dominated by the dictionary creation in Steps 2--7. We create a list for each dictionary key, and there are at most $n$ keys. The sum of sizes of all lists in the dictionary is equal to the number of elements observed, i.e., $n$. So the memory, assuming each element of $[n]$ can be stored in constant space, is $\order(n)$.

\medskip

Having proved that the \textsc{WingIt} estimator can be computed using linear time and space, we now turn to studying its estimation error properties.

\section{Theoretical results on missing mass estimation} \label{sec:theory}

This section presents a risk guarantee for $\Mhat_{\some}$ and a variance bound on the estimand $M_{\pi}(X^n)$.

\subsection{Risk guarantee for the \textsc{WingIt} estimator}

We first provide a guarantee for the risk of the estimator $\Mhat_{\some}$~\eqref{eq:estimator-some}. Recall the definition~\eqref{eq:tmix} of the mixing time up to arbitrary total variation $\tmix(\epsilon)$ and that we write $\Tmix = \tmix(1/4)$.

\begin{theorem} \label{thm:main-upper-bound}
Suppose we choose $\T \geq \tmix\left( (\Tmix/n) \land 1/4 \right)$, and let $\Mhat_{\some}(\T)$ denote the estimator defined in Eq.~\eqref{eq:estimator-some}. Then there is an absolute positive constant $C$ such that
\begin{align} \label{main-theorem-bd}
\MSE(\Mhat_{\some}(\T), M_{\pi}) \leq C \cdot \frac{\T}{n} \land 1.
\end{align}
\end{theorem}
In the special case of i.i.d sequences, the chain mixes in one step, our estimator specializes to Good--Turing, and we may set $\T = 1$ in Theorem~\ref{thm:main-upper-bound} to recover existing guarantees for the Good--Turing estimator of missing mass on i.i.d. sequences~\citep{mcallester2000convergence,rajaraman2017minimax}.
A few remarks on the general Markov case are now in order. We assume that $n \geq 4\Tmix$ in all the remarks below; if $n < 4\Tmix$, then the RHS in Eq.~\eqref{main-theorem-bd} reduces to a universal constant and conversely, it is straightforward to show that consistent estimation is impossible (see footnote 3 below). Let us now proceed to our discussion.

First, observe that if $n \geq 4\Tmix$, it follows from the mixing condition in Eqs.~\eqref{eq:tmix} and~\eqref{eq:tmix-logs} that $\tmix\left( (\frac{\Tmix}{n}) \land 1/4 \right) \leq \Tmix \cdot \log (n/\Tmix)$.
Therefore, setting the window size $\tau \asymp \Tmix \cdot \log (n/\Tmix)$, we obtain  
\begin{align*}
\MSE(\Mhat_{\some}(\T), M_{\pi}) \lesssim \frac{\Tmix}{n} \cdot \log (n/\Tmix),
\end{align*}
so that the MSE is on the order $\mathcal{O}\left(\frac{\log n^*}{n^*}\right)$ with $n^* = n/\Tmix$ denoting the effective sample size.
Note that by Markov's inequality, we immediately obtain that for all $\epsilon > 0$,
\begin{align*}
\Pr\left\{ |\Mhat_{\some}(\T) - M_{\pi}(X^n)| \geq \epsilon \right\} \leq C \frac{\Tmix}{n \epsilon^2} \cdot \log (n/\Tmix),
\end{align*}
thereby showing that $\Mhat_{\some}(\T)$ is a consistent estimator of $M_{\pi}$ provided $\T$ is chosen appropriately. 

It is worth remarking at this juncture that the window size $\T$ is a hyperparameter in our algorithm, and Theorem~\ref{thm:main-upper-bound} holds provided it is chosen in a data-independent fashion but satisfies the (non-random) bound $\T \geq \tmix\left( (\Tmix/n) \land 1/4 \right)$.
 In practice, we may not know  $\Tmix$ (even up to a constant factor) and one may need to tune $\tau$ in a data-dependent manner. In this case, Theorem~\ref{thm:main-upper-bound} does not apply as is, but our proof techniques may still be useful in showing that a data-dependent procedure is valid. Note that estimators for the mixing time are available in the literature, e.g.~for reversible Markov chains~\citep{hsu2019mixing}, and these mixing time estimates could be used to tune $\T$. We sketch a different data-dependent tuning method in Section~\ref{sec:expts}. Theoretically analyzing such estimators with data-dependent $\tau$ is an important direction for future work.

Next, we remark on the issue of optimality. \citet{chandra2022missing} proved a minimax lower bound of order $\Omega((np)^{-1})$ on the mean squared error of estimating missing mass in sticky chains of the form~\eqref{eq:sticky-chain}. As remarked on before (see Lemma~\ref{lem:stickymixingtime}), such chains have a mixing time $\Tmix \asymp p^{-1}$, so this yields a lower bound of $\Omega(\Tmix/n)$ for such chains. 
Theorem~\ref{thm:main-upper-bound} matches this lower bound up to the logarithmic factor $\log(n/\Tmix)$ and further holds for \emph{all} chains of mixing time\footnote{For the specific class of sticky chains, Theorem~\ref{thm:main-upper-bound} would yield the rate of $\mathcal{O}((np)^{-1} \log (np))$, which matches the lower bound of~\citet{chandra2022missing} up to a factor $\log(np)$.} at most $\Tmix$. 
More formally, define the class of Markov chains that mix in time at most $T$, as
\begin{align*}
\mathcal{P}_{\mathsf{mix}}(T) := \{\text{Markov chain } (\bm{P}, \pi): \text{ mixing time of chain } \Tmix \text{ is at most } T \}.
\end{align*}
Theorem~\ref{thm:main-upper-bound} implies that for a universal constant $C > 0$, we have the worst-case upper bound
\begin{subequations}
\begin{align} \label{eq:worst-case-ub}
\sup_{(\Pmat, \pi) \in \mathcal{P}_{\mathsf{mix}}(T)} \; \MSE(\Mhat_{\some}(2 T \log n), M_{\pi}) \leq C\cdot \frac{T \log (n/T)}{n}.
\end{align}
On the other hand, we may state\footnote{Such a statement follows from noting that in~\citet[Eq. (6)]{chandra2022missing}: (a) We can set $T = 2/(1-\alpha)$, and (b) When $n \geq \frac{2\log n}{1 - \alpha} = 2 T \log n$, the second term on the RHS can be made less than half the first term.} the minimax lower bound~\citep[Theorem 2]{chandra2022missing} as the following: There is a universal constant $c > 0$ such that if $n \geq 2 T \log n$ then for any estimator $\Mhat$ that is a measurable function of the observations $X^n$, we must have 
\begin{align} \label{eq:minimax-lb}
\sup_{(\Pmat, \pi) \in \mathcal{P}_{\mathsf{mix}}(T)} \; \MSE(\Mhat, M_{\pi}) \geq c \cdot \frac{T}{n}.
\end{align}
\end{subequations}
Taken together, Eqs.~\eqref{eq:worst-case-ub} and~\eqref{eq:minimax-lb} thus imply that in the regime\footnote{In the regime $n \leq \Tmix$, the worst case risk of any estimator can be shown to be lower bounded by a constant, so our estimator is also trivially minimax-optimal in this regime.} $n \gtrsim T \log n$, the \textsc{WingIt} estimator is information-theoretically minimax optimal up to a logarithmic factor in $n$. Removing this logarithmic factor is an interesting open problem, and will likely require new ideas both in terms of algorithm design and analysis. 

Finally, we comment on our analysis path, which is significantly different from the related literature on missing mass estimation from an i.i.d. sequence.
As alluded to before, a natural and popular method to analyze estimators of the missing mass in the i.i.d. setting~\citep{mcallester2000convergence,mcallester2001learning} is to exploit concentration of the estimand and write
\begin{align} \label{eq:three-terms}
\MSE(\Mhat, M_{\pi}) \leq 3 | \EE_{X^n} [\Mhat(X^n)] - \EE_{X^n} [M_{\pi}(X^n)] |^2 + 3 \Var( \Mhat(X^n)) + 3 \Var( M_{\pi}(X^n)),
\end{align}
which can be obtained by adding and subtracting terms and using the elementary inequality $(a + b + c)^2 \leq 3 (a^2 + b^2 + c^2)$. Operationally, therefore, analyzing the MSE of the estimator relies in itself on understanding the variance of the missing mass random variable $M_{\pi}(X^n)$, which has nothing to do with the estimator. 
In the Markovian case, it appears  challenging to control $\Var( M_{\pi}(X^n))$ by straightforward means.  Other analysis techniques for missing mass estimation in the i.i.d. setting~\citep[e.g.][]{rajaraman2017minimax,chandra2019improved} work with an exact decomposition of the MSE expressed as a sum of weighted indicators over pairs of elements $x, x' \in \Xspace$, and use the i.i.d. assumption to bound these terms in a precise fashion. One such property that is used to show concentration is the negative associativity of certain random variables~\citep{mcallester2003concentration}, and
we do not expect this property to hold for general Markov chains.

In contrast to these approaches, we begin with a nonstandard decomposition of the MSE by conditioning on the sequence $X^n$, and our argument deviates significantly from Eq.~\eqref{eq:three-terms}.
Additionally, owing to the structure of our estimator~\eqref{eq:estimator-some}, we must compare our random sequence $X^n$ to suitably modified random sequences with windows of random variables left out and/or replaced by independent copies; we do so by proving certain total variation bounds for Markov bridges, in Lemmas~\ref{lem:Markov-bridge} and~\ref{lem:two-bridges}. 

\subsection{Variance of missing mass functional $M_{\pi}(X^n)$}\label{sec:missingmassvariance}

The analysis path that we sketched above circumvents needing to control the variance of the estimand, i.e. $\text{var}(M_{\pi}(X^n))$.
Nevertheless, and somewhat surprisingly, analyzing various properties of the estimator $\Mhat_{\some}$ allows us to indirectly upper bound $\text{var}(M_{\pi}(X^n))$. We state this result as the following theorem, which could be of independent interest.

\begin{theorem} \label{thm:variance-bound-MM}
There is an absolute positive constant $C$ such that
\begin{align} \label{eq:var-bound-thm2}
\Var(M_\pi(X^n)) \leq C \cdot \frac{\Tmix^2 \cdot \log (1 + n/\Tmix)}{n} \land 1.
\end{align}
\end{theorem}
Theorem~\ref{thm:variance-bound-MM} is proved in Section~\ref{sec:var-proof}.
En route, we control the variance of the estimator $\Mhat_{\some}(\T)$ by proving that it satisfies a certain stability (i.e. bounded differences) property for all values of $\T$---see Lemma~\ref{lem:bdd-diff}, which may be of independent interest. Intuitively speaking, the random variable $\Mhat_{\some}(\T)$ satisfies a bounded differences property with respect to the sequence $X^n$ since when $\T$ is small the impact of changing one coordinate is local\footnote{In the extreme case where the window length is $1$, we recover the well-known bounded differences property of the vanilla Good--Turing estimator~\citep{mcallester2000convergence}.}. Having said that, we conjecture that the bound~\eqref{eq:var-bound-thm2} can be improved by replacing $\Tmix^2$ by $\Tmix$, but doing so will require different techniques since the bounded differences inequality in Lemma~\ref{lem:bdd-diff} is tight (see Remark~\ref{rem:bdd-diff-tight}).

The most related result to Theorem~\ref{thm:variance-bound-MM} was proved by~\citet{skorski2020missing}, who showed a \emph{one-sided} tail bound on $M_{\pi}(X^n)$ in terms of the hitting time of large sets of the Markov chain. Even though the hitting time of large sets is comparable to the mixing time~\citep{oliveira2012mixing, peres2015mixing}, the main result of~\citet{skorski2020missing} cannot, strictly speaking, be compared with Theorem~\ref{thm:variance-bound-MM}. On the one hand, Theorem~\ref{thm:variance-bound-MM} implies the two-sided polynomial tail bound
\begin{align*}
    \Pr\left\{ |M_{\pi}(X^n) - \EE[M_{\pi}(X^n)] | 
 \geq \epsilon \right\} \leq C \cdot \frac{\Tmix^2 \log (1 + n/\Tmix)}{n \epsilon^2} \text{ for all } \epsilon > 0, 
\end{align*}
which can be obtained via direct application of 
Markov's inequality. On the other hand,~\citet[Corollary 1]{skorski2020missing} provides a stronger exponentially decaying bound depending linearly on $\Tmix$, but only on the upper tail of the random variable and without centering it at its expectation.
Consequently, a variance bound cannot be extracted from that result. 

Having discussed estimation guarantees for the missing mass, we now turn to the problem of estimating small-count probabilities.

\section{Estimating stationary mass of elements with frequency at most~$\zeta$} \label{sec:extension}

In this section, we show that the idea behind the $\some$ estimator can be applied robustly to estimate not only the missing mass functional $M_{\pi}(X^n)$, but also the mass of all elements that occur \emph{at most} $\zeta$ times~\eqref{eq:small-count}.
To define this functional formally, let $N_x(X_P) = N_x(\bm{X}_P) := \sum_{i \in P} \ind{X_i = x}$ denote the number of occurrences of the element $x \in \Xspace$ in the (sub)-sequence $X_P$ and (sub)-set $\bm{X}_P$.
Then, the mass of all elements that occur \emph{exactly} $\zeta$ times is defined as
\begin{align}\label{eq:mass-zeta}
    M_{\pi,\zeta}(X^n) := \sum_{x \in \Xspace} \pi_x \cdot \ind{N_x(X^n) = \zeta}.
\end{align}
We focus on the mass of all elements that occur \emph{at most} $\zeta$ times (cf. the definition in Eq.~\eqref{eq:small-count})
\begin{align}\label{eq:mass-atmost-zeta}
    M_{\pi,\leq \zeta}(X^n) := \sum_{x \in \Xspace} \pi_x \cdot \ind{N_x(X^n) \leq \zeta}.
\end{align}
Clearly, both Eq.~\eqref{eq:mass-zeta} and Eq.~\eqref{eq:mass-atmost-zeta} recover the missing mass functional $M_{\pi}(X^n)$ when $\zeta = 0$.
For small $\zeta$, both of these functionals provide more fine-grained information about the mass placed by the stationary distribution on low-frequency elements of $X^n$. 
It is worth noting that the functional~\eqref{eq:mass-atmost-zeta} is directly related to the discovery probability~\citep{lijoi2007bayesian,favaro2012new}, whereby we are interested in the probability of ``discovering" in the test sample an element that appeared rarely (i.e. $\leq \zeta$ times) in the training sample .

We now define a natural extension of the $\some$ estimator for the functional $M_{\pi, \leq \zeta}(X^n)$ that retains the leave-a-window-out principle.
In particular, recalling our notation from Eq.~\eqref{eq:index-sets}, we generalize the missing mass estimator $\Mhat_{\tau}^{(i)}$~\eqref{eq:Mi} via
\begin{align} \label{eq:estimator-some-atmostzeta}
    \Mhat_{\tau,\leq \zeta}^{(i)} := \ind{N_{X_i}(\bm{X}_{\Isetone_i}) \leq \zeta}, \;\; \text{ and construct the estimator } \;\;    \Mhat_{\some,\leq \zeta}(\tau) := \frac{1}{n} \sum_{i=1}^{n} \Mhat_{\tau,\leq \zeta}^{(i)}.
\end{align}
The following theorem shows that the estimator $\Mhat_{\some, \leq \zeta}(\tau)$ has small MSE.

\begin{theorem}\label{thm:main-upper-bound-zeta}
    Suppose we choose $\T \geq \tmix\left( (\Tmix/n) \land 1/4 \right)$, and let $\Mhat_{\some, \leq \zeta}(\T)$ denote the estimator defined in Eq.~\eqref{eq:estimator-some-atmostzeta}. Then there is an absolute positive constant $C$ such that
\begin{align}\label{eq:main-upper-bound-zeta}
\MSE(\Mhat_{\some, \leq \zeta}(\T), M_{\pi, \leq \zeta}) \leq C \cdot \frac{(\zeta + 1) \T}{n} \land 1.
\end{align}
\end{theorem}
Theorem~\ref{thm:main-upper-bound-zeta} is proved in Section~\ref{sec:main-proof-zeta}; a few remarks on the result follow.
First, note that by setting $\zeta = 0$, Theorem~\ref{thm:main-upper-bound-zeta} recovers Theorem~\ref{thm:main-upper-bound}, our result for missing mass, since the estimator $\Mhat_{\some, \leq 0}(\T)$ exactly coincides with the missing mass estimator $\Mhat_{\some}(\T)$. Accordingly, the proof of this theorem generalizes (and is structured similarly to) the proof of Theorem~\ref{thm:main-upper-bound}. Second, note that 
setting $\T \asymp \Tmix \log(1 + n/\Tmix)$ yields
\begin{align} \label{eq:small-count-unrolled}
\MSE(\Mhat_{\some, \leq \zeta}(\T), M_{\pi, \leq \zeta}) \lesssim \frac{(\zeta + 1) \Tmix}{n} \cdot \log(1 + n/\Tmix),
\end{align}
so that the MSE is on the order $\mathcal{O}\left((\zeta + 1) \cdot \frac{\log n^*}{n^*}\right)$ with $n^* = n/\Tmix$ denoting the effective sample size.
To our knowledge, a result equivalent to Eq.~\eqref{eq:small-count-unrolled} with linear dependence on $\zeta + 1$ is not directly available in the literature, even in the i.i.d. setting. 
In fact, it is instructive to revisit the i.i.d. setting for small-count probabilities and compare with the Good--Turing estimator~\citep{good1953population}.
\begin{remark} \label{rem:WingIt-GT}
     Recalling the notation $\phi_s(X^n)$ for the number of elements of the sample space $\Xspace$ that occur $s$ times in $X^n$, the Good--Turing estimator for the functional $M_{\pi, \zeta}(X^n)$~\eqref{eq:mass-zeta} is given by 
$\Mhat_{\mathsf{GT}, \zeta} = \frac{\zeta + 1}{n} \cdot \phi_{\zeta + 1} (X^n)$. 
We can then derive an estimator for $M_{\pi, \leq \zeta}(X^n)$ by writing  $\Mhat_{\mathsf{GT}, \leq \zeta} = \sum_{s = 0}^{\zeta} \Mhat_{\mathsf{GT}, s}$. Conversely, we can derive an estimator of the exact-count functional $M_{\pi, \zeta}$ from our estimator $\Mhat_{\some,\leq \zeta}$. In particular, 
we can construct the estimator
    $\Mhat_{\some,\zeta}(\tau) := \Mhat_{\some,\leq \zeta}(\tau) - \Mhat_{\some,\leq (\zeta - 1)}(\tau)$, which can be  interpreted as writing
\begin{align*}
    \Mhat_{\tau, \zeta}^{(i)} := \ind{N_{X_i}(\bm{X}_{\Isetone_i}) = \zeta}, \;\; \text{ and constructing the estimator } \;\;    \Mhat_{\some,\zeta}(\tau) = \frac{1}{n} \sum_{i=1}^{n} \Mhat_{\tau, \zeta}^{(i)}.
\end{align*}
The leave-one-out perspective~\citep{mcallester2001learning} then yields the following consequence for $\T = 1$ in our estimator: We have $\Mhat_{\some,\zeta}(1) = M_{\mathsf{GT}, \zeta}$, and therefore $\Mhat_{\some, \leq \zeta}(1) = M_{\mathsf{GT}, \leq \zeta}$.
\end{remark}

To our knowledge, existing analyses of the Good--Turing estimator that are tailored to exact-count estimation do not recover the small-count estimation error guarantee of Theorem~\ref{thm:main-upper-bound-zeta} even in the special case of i.i.d. observations; simply translating these results to guarantees on estimating $\Mhat_{\mathsf{GT}, \zeta}$ leads to weaker guarantees. To be concrete, 
applying the result of~\citet{drukh2005concentration} (which is for the MSE of $\Mhat_{\mathsf{GT}, \zeta}$) yields
\begin{subequations} \label{eq:GT-upto-iid}
\begin{align} \label{eq:DM}
\MSE(\Mhat_{\mathsf{GT}, \leq \zeta}, M_{\pi, \leq \zeta}) \overset{\1}{\leq} \zeta \sum_{s = 0}^{\zeta} \MSE(\Mhat_{\mathsf{GT},s}, M_{\pi, s}) = \zeta \sum_{s = 0}^{\zeta}  \frac{\sqrt{s}}{n} + \left(\frac{s}{n}\right)^2 \lesssim \frac{\zeta^{5/2}}{n} \land 1,
\end{align}
where step $\1$ follows from the (loose) inequality $(\sum_{s = 0}^{\zeta} a_{s})^2 \leq \zeta \sum_{s = 0}^{\zeta} a^2_{s}$.
However, setting $\T = 1$ in Theorem~\ref{thm:main-upper-bound-zeta}, we see that Eq.~\eqref{eq:main-upper-bound-zeta} improves the guarantee~\eqref{eq:DM} even in the i.i.d. case, showing that 
\begin{align} \label{eq:small-count-iid-improved}
\MSE(\Mhat_{\mathsf{GT}, \leq \zeta}, M_{\pi, \leq \zeta}) \overset{\1}{=}  \MSE(\Mhat_{\some, \leq \zeta}(1), M_{\pi, \leq \zeta}) \lesssim (\zeta + 1) / n,
\end{align}
\end{subequations}
where step $\1$ follows from Remark~\ref{rem:WingIt-GT}. 

Conversely, our bound on the exact-count Good--Turing estimator (obtained by setting $\T = 1$ in Theorem~\ref{thm:main-upper-bound-zeta} and appealing to Remark~\ref{rem:WingIt-GT}) is suboptimal in its dependence on $\zeta$. Our bound specializes in the i.i.d. case to
\begin{subequations} \label{eq:GT-bounds-exact-iid}
\begin{align} \label{eq:GT-exact-iid}
\MSE(\Mhat_{\mathsf{GT}, \zeta}, M_{\pi, \zeta}) \leq 2\MSE(\Mhat_{\mathsf{GT}, \leq \zeta}, M_{\pi, \leq \zeta}) + 2\MSE(\Mhat_{\mathsf{GT}, \leq (\zeta -1)}, M_{\pi, \leq (\zeta -1)}) \lesssim \frac{(\zeta + 1)}{n}.
\end{align}
Note that the dependence on $\zeta$ is linear, compared to the following bound of~\citet{drukh2005concentration} that has an improved dependence on $\zeta$:
\begin{align}
    \MSE(\Mhat_{\mathsf{GT}, \zeta}, M_{\pi, \zeta}) \lesssim \frac{\sqrt{\zeta + 1}}{n} + \left( \frac{\zeta + 1}{n} \right)^2.
\end{align}
\end{subequations}
It is worth noting that the bounds~\eqref{eq:GT-upto-iid} are on the same estimator, and the bounds~\eqref{eq:GT-bounds-exact-iid} are on the same estimator. 
Our analysis technique appears to be better equipped to deal with estimation error on the small-count probabilities, i.e. $\MSE(\Mhat_{\mathsf{GT}, \leq \zeta}, M_{\pi, \leq \zeta})$, while the analysis technique of~\cite{drukh2005concentration} is better equipped to deal with estimation error on the exact-count probabilities, i.e., $\MSE(\Mhat_{\mathsf{GT}, \zeta}, M_{\pi, \zeta})$.

The problems of whether the rate~\eqref{eq:small-count-iid-improved} and its Markovian analog~\eqref{eq:small-count-unrolled} are information-theoretically optimal for estimating the small-count\footnote{Note that the small $\zeta$ regime, i.e. $\zeta = o(n)$, is the interesting one; we should expect accurate estimation to be possible for large $\zeta$ since the corresponding elements appear many times.} probability $M_{\pi, \leq \zeta}$ are interesting and, to our knowledge, open,  both in the i.i.d. and Markovian settings. 
To address this, it would be interesting to examine and carefully modify generalizations of Good--Turing estimators (e.g.~\cite{painsky2023generalized}) for the Markov setting.

\section{Numerical experiments} \label{sec:expts}

In this section, we provide a set of simulations on synthetically constructed Markov chains and on natural language text in order to corroborate our theoretical results, in particular Theorem~\ref{thm:main-upper-bound}.
Before proceeding to the experiments themselves, we describe in Section~\ref{sec:expts-tuning} a data-dependent tuning procedure of the window size $\T$ in the estimator $\Mhat_{\some}(\T)$. 

Code and the text used for the simulations are available at \citet{githubmissingmass}.

\subsection{Data-dependent tuning of window size $\T$}\label{sec:expts-tuning}

Theorem~\ref{thm:main-upper-bound} prescribes that we choose the window size $\T$ to be at least on the order $\Tmix \log(1 + n/\Tmix)$. An important question to address is how to choose the window size $\T$ when we do not have access to a valid upper bound on $\Tmix$.
We now propose a validation procedure to select $\T$, and test this procedure in the experiments in the sequel. For this section, assume $n$ is divisible by $3$ for notational convenience.

Given the sequence $X^{n}$, we choose a candidate window size $\widehat{\T}$ via the following procedure. We first split the sequence into the first one-third $Z^{(1)} = (X_1, \ldots, X_{n/3})$ and the final one-third $Z^{(2)} = (X_{2n/3 + 1}, \ldots, X_n)$ and compute the random variable
\[
\widetilde{M}(Z^{(1)},Z^{(2)}) = \frac{1}{(n/3)} \sum_{i = 2n/3 + 1}^{n} \ind{X_{i} \notin \{X_1, \ldots, X_{n/3} \} }.
\]
Next, iterate $\T = 1, 2, 4, \ldots,  2^{\lfloor \log_2 (n/6) \rfloor}$ in increasing order, and compute $\Mhat_{\some}(\tau)$ on the sequence $Z^{(1)}$; denote this random variable by $\Mhat_{\some}(Z^{(1)}; \tau)$ for convenience. We then set $\widehat{\T}$ to be the smallest $\T$ among this set such that
\begin{align} \label{eq:tuning-T-choice}
\left| \Mhat_{\some}(Z^{(1)}; \tau) - \widetilde{M}(Z^{(1)},Z^{(2}) \right|^2 \leq \frac{C_{\mathsf{tune}}\, \T}{(n/3)}
\end{align}
for a suitable choice of the constant $C_{\mathsf{tune}} > 0$. If such an inequality is not satisfied for any $\T$ in the prescribed list, then we set $\widehat{\tau} = n/6$. While we do not prove theoretical guarantees for the tuned estimator, Appendix~\ref{sec:tuning-justification} provides some intuition for why this procedure is reasonable as an automatic tuning method. We next show that empirically, this tuning method is competitive with the optimally chosen window size.

\subsection{Experiments on simulated Markov chains and natural language text} 

For the simulated Markov and natural language text sequences considered in this section,
we vary the sequence length $n$ and plot the MSE of the estimator $\Mhat_{\some}(\T)$ as a function of $n$ for different values of the window size $\T$. 
We also plot the MSE of the tuned estimator $\Mhat_{\some}(\widehat{\T})$.
Note once again that the special case $\T = 1$ corresponds to the Good--Turing estimator~\eqref{eq:GT-orig}. We also plot (in dashed lines) the result of the tuning procedure that we described in Section~\ref{sec:expts-tuning}, with the constant $C_{\mathsf{tune}}$ in Eq.~\eqref{eq:tuning-T-choice} set to $1$. Every point in these plots is generated by averaging the results of multiple sequences generated from the source.
Throughout this section we denote $\Mhat(\cdot) := \Mhat_{\some}(\cdot)$ as shorthand.

First, and as a sanity check, we consider the case of the trivial Markov chain formed by i.i.d. samples. 
For generating $n$ samples, we consider the uniform distribution over the state space $\Xspace=\{1,2,\ldots,\lfloor 1.2n\rfloor\}$, which is close to the worst-case distribution for the Good--Turing estimator. Moreover, 
this ensures that the missing mass $M_{\pi}(X^n)$ is significant.
Figure~\ref{fig:iidunif} shows two plots.  
\begin{figure}
    \centering
    \resizebox{\linewidth}{!}{\input{iiduniform.pgf}}
    \caption{IID Uniform$([1.2n])$ for $n$ samples, averaged over 100 trajectories.}
    \label{fig:iidunif}
\end{figure}
The plot on the left is that of the MSE of the estimator $\Mhat_{\some}(\T)$ as a function of sequence length $n$ for various values of the window size $\T$ and for tuned $\widehat{\T}$. The plot on the right shows the mean values (over the 100 runs) of the triple of random variables $(M_{\pi}$, $\Mhat(1), \Mhat(\widehat{\T}))$ along with the 90 percentile confidence bar (5th to 95th percentile). 
Observe that on the one hand---and as expected in the i.i.d. setting with mixing time $\Tmix = 1$---the minimum MSE is attained when $\T = 1$, i.e., by the vanilla Good--Turing estimator~\eqref{eq:GT-orig}. On the other hand, the MSE is only marginally higher for higher values of the window size $\T$, and all estimators appear to enjoy the same rate of decay of MSE in the sequence length $n$. The effect of misspecifying $\T$ appears to become insignificant as $n$ increases. The MSE of $\Mhat(\widehat{\T})$ with tuned $\widehat{\T}$ (dashed line) is close to the minimum attained MSE. 
In the plot to the right, the mean values of missing mass $M_{\pi}$ and the estimators $\Mhat(1)$, $\Mhat(\widehat{\T})$ are almost overlapping for all $n$. The confidence bars for the three quantities are shown as a wide blue bar for $M_{\pi}$, an orange bar for $\Mhat(1)$ and as a capped black line for $\Mhat(\widehat{\T})$. The confidence bars, narrow to begin with, shrink to negligible lengths as $n$ increases.

In our second experiment, we once again consider the state space $\Xspace = \{1,\ldots, \lfloor 1.2n \rfloor\}$. We simulate the sticky Markov chain~\eqref{eq:sticky-chain} with $p = 0.5$ and stationary distribution given by the uniform distribution on $\Xspace$, i.e., $\pi_x = \frac{1}{\lfloor 1.2n \rfloor}$ for all $x \in \Xspace$. As before, we simulate the performance of the $\textsc{WingIt}$ estimator as a function of $n$ for different values of window size $\T$, and present our results in Figure~\ref{fig:stick0.5}.
\begin{figure}
    \centering
    \resizebox{\linewidth}{!}{\input{sticky0.5.pgf}}
    \caption{Sticky$(0.5)$ with $\pi =$ Uniform$([1.2n])$ for $n$ samples, averaged over 100 trajectories.}
    \label{fig:stick0.5}
\end{figure}
In the MSE plot to the left, we observe the following. In contrast to the previous i.i.d. example, we now find that the choice $\T = 1$ is a poor one, and that the error of the estimator does not decay with the sample size $n$. As expected from the analysis presented in Section~\ref{sec:iid}, the vanilla Good--Turing estimator~\eqref{eq:GT-orig} indeed suffers a constant MSE in this setting. Note that by Lemma~\ref{lem:stickymixingtime}, the mixing time of this chain is bounded as $\Tmix \in [1, 4]$, and Theorem~\ref{thm:main-upper-bound} predicts that the estimator should succeed when the window size $\T$ is larger than $\Tmix$. This prediction is borne out in simulation: While the estimator exhibits constant bias even when $\T = 4$, when $\T = 8$ the MSE suddenly decays in $n$. Further increases in the window size $\T$ preserve the consistency of the estimator while affecting the MSE only slightly.
In the plot to the right in Figure~\ref{fig:stick0.5}, we see that the mean value of missing mass coincides with the mean value of both the estimators (one with $\T=16$ and the other with tuned window size) for larger values of $n$. For $n<2000$, the mean of the estimator $\Mhat_{\some}(16)$ almost coincides with missing mass, while the mean of $\Mhat_{\some}(\widehat{\T})$ is smaller. The confidence bars for both estimators are wider than that of the missing mass for $n<2000$  though they narrow quickly as $n$ grows. These observations suggest that the tuning procedure may be improvable, particularly for small $n$. 

\begin{figure}[ht!]
    \centering
    \resizebox{\linewidth}{!}{\input{sticky.pgf}}
    \caption{Sticky$(\overline{p})$ with $\pi =$ Uniform$([1.2n])$ for $n$ samples, 100 trajectories. Compared to Eq.~\eqref{eq:sticky-chain}, we have reparameterized as $\overline{p} = 1 - p$. To reduce clutter, confidence bar is shown only for $\Mhat_{\some}(\widehat{\T})$.}
    \label{fig:sticky}
\end{figure}
In our next experiment, we simulate the sticky Markov chain~\eqref{eq:sticky-chain} with $\overline{p}\triangleq1-p = 0.1,0.5,0.9$ setting the state space and stationary distribution as before. 
The aim of this experiment is to examine more closely the influence of the stickiness parameter $\overline{p}$ on the optimal choice of window size $\T$.
The same simulations are performed and results are compared in Figure~\ref{fig:sticky}. 
The best (power of 2) window sizes for $\overline{p}=0.1,0.5,0.9$ are observed through  simulations to be, respectively, $\T=4,16,64$. The plot to the left shows the MSE of $\Mhat_{\some}(\T)$ for the best observed window size and for the data-tuned window size $\widehat{\T}$ (called ``tuned"). For the least sticky chain ($\overline{p}=0.1$), the tuned and the best estimators have overlapping MSEs for all $n$. As stickiness increases, the MSE of the tuned estimator marginally deviates from the best for small $n$. For highly sticky chains ($\overline{p}=0.9$), the deviation persists for significantly longer. The plot to the right shows mean values and confidence bars (for the tuned estimator alone) for all three values of $\overline{p}$. The best estimator is close in mean to missing mass for all $n$ for all $\overline{p}$. As stickiness increases, the tuning procedure appears to require increasingly larger values of $n$ for good accuracy.

In our final experiment, presented in Figure~\ref{fig:tale}, we consider the text of the novel \emph{A Tale of Two Cities} by Charles Dickens accessed through Project Gutenberg~\citep{Dickens:1994:TTC}.
All auxiliary content (preface, table of contents, chapter titles, Project Gutenberg related text) were removed and only the novel text was retained. This text was tokenized and all punctuation was removed. Titles (Miss, Mr. etc.) and names of characters that occurred as collocations with high frequency (10 of them) were merged into single tokens. The result was a sequence of $N=136092$ tokens, numbered from $0$ to $N-1$, with a vocabulary $\Xspace$ (unique tokens) of size $|\Xspace|=10542$. For defining missing mass $M_{\pi}$, the overall frequency distribution of the $N$ tokens was taken to be the stationary distribution $\pi$. A consecutive sequence of $n$ tokens (Token $s+1$ to Token $s+n$ with starting point $s$) is considered as a trajectory of length $n$, i.e. $X^n$. For a given length $n$, approximately $15N/n$ trajectories were considered in the simulations with their starting points separated by $n/15$. An important feature of this data is that the Markov assumption (also known as the bigram assumption in this literature) is clearly violated, since we expect the text to have longer-range dependencies. In any case, we can run our estimator for the missing mass and compare it to the true missing mass, which can still be computed from the sequence once the stationary probability is fixed.

\begin{figure}[ht!]
    \centering
    \resizebox{\linewidth}{!}{\input{taleoftwocities.pgf}}
    \caption{Text of `A Tale of Two Cities' by Charles Dickens, Vocabulary: 10542 words, Trajectories: sequences of $n$ words from text.}
    \label{fig:tale}
\end{figure}
In Figure~\ref{fig:tale}, we show the MSE in the plot to the left and the means with confidence bars to the right. Interestingly, all of the choices of $\T$ that we consider yield similar MSE performance for the $\some$ estimator, and \emph{all} of these choices appear to have a super-linear rate of decay with the sample size $n$. The reason for the difference in rate of decay could be because we hold $|\Xspace|$ constant as we increase $n$ in the text simulations resulting in a decrease in $M_{\pi}$ with $n$. This decrease is confirmed in the plot on the right, where we see that the mean of missing mass falls from about $0.45$ for $n=600$ to about $0.1$ for $n=19200$. Further, from the right plot, we observe that the tuned estimator and the estimator with $\T=32$ are close to each other in mean, while deviating a bit from the mean of missing mass for small $n$. The aforementioned long-range dependencies in the text are likely to be causing this minor deviation. In any case, the estimator $\Mhat_{\some}(\widehat{\T})$ appears to be accurate for all $n$. Overall, our experiments on this corpus demonstrate that the missing mass estimator is robust to model misspecification, and could work well even for non-Markovian sources.

\section{Proofs} \label{sec:proofs}
\vspace{-2mm}

In this section, we present proofs of the main theorems. We begin with preliminaries in Section~\ref{sec:pf-prelim}, which introduces additional notation and a useful reduction device that will help us analyze our estimators $\Mhat_{\some}(\tau)$ and $\Mhat_{\some, \leq \zeta}(\tau)$. Technical lemmas are frequently referenced in our proofs, and their statements and proofs can be found in Appendix~\ref{sec:lemmas}.

\subsection{Preliminary decompositions and notation} \label{sec:pf-prelim}

Suppose for convenience\footnote{Our argument extends straightforwardly without this assumption; we only make it to avoid carrying floor and ceiling notation.} that $n$ is divisible by $2\tau$, and let $n_0 = n/(2\T)$. 
Recall our single-sample estimators $\Mhat_{\T}^{(i)}$ (Eq.~\eqref{eq:Mi}) and the definition~\eqref{eq:estimator-some} of the \textsc{WingIt} estimator. Define the ``skipped" estimators
\begin{align} \label{eq:estimator-skipped}
\Mhat_{\some}(\T; \ell) := \frac{1}{n_0} \sum_{j = 1}^{n_0} \Mhat_{\T}^{(2\T j - \ell)} \text{ for each } \ell = 0, \ldots, 2\T - 1.
\end{align}
In words, each of these estimators averages only $n_0$ of the individual estimates $\Mhat^{(i)}_{\T}$ by skipping $2\T$ indices at a time; this skipping induces further decorrelations.
Note that we may write
\[
\Mhat_{\some}(\T) = \frac{1}{2\tau} \sum_{\ell = 0}^{2\tau - 1} \Mhat_{\some}(\T; \ell)
\]
by definition. Furthermore, we have
\begin{align} \label{eq:skip-decomp}
\MSE(\Mhat_{\some}(\T), M_{\pi}) &= \EE \left( \frac{1}{2\T} \sum_{\ell = 0}^{2\T - 1} \Mhat_{\some}(\T; \ell) - M_{\pi} \right)^2 \notag \\
&\overset{\1}{\leq} \frac{1}{2\tau} \sum_{\ell = 0}^{2\T- 1} \EE \left( \Mhat_{\some}(\T; \ell) - M_{\pi} \right)^2,
\end{align}
where step $\1$ follows from Jensen's inequality applied to the convex function $z \mapsto z^2$.

Via a parallel argument, and introducing the objects
\begin{align} \label{eq:estimator-skipped-smallcount}
\Mhat_{\some, \leq \zeta}(\T; \ell) := \frac{1}{n_0} \sum_{j = 1}^{n_0} \Mhat_{\T, \leq \zeta}^{(2\T j - k)} \text{ for each } \ell = 0, \ldots, 2\T - 1,
\end{align}
we have
\begin{align} \label{eq:skip-decomp-zeta}
\MSE(\Mhat_{\some, \leq \zeta}(\T), M_{\pi, \leq \zeta}) \leq \frac{1}{2\tau} \sum_{\ell = 0}^{2\tau - 1} \EE \left( \Mhat_{\some, \leq \zeta}(\T; \ell) - M_{\pi, \leq \zeta} \right)^2.
\end{align}

Our argument to prove Theorems~\ref{thm:main-upper-bound} and~\ref{thm:main-upper-bound-zeta} will proceed by establishing the following proposition.
\begin{proposition} \label{prop:skipped-reduction}
If $\T \geq \tmix\left( (\Tmix/n) \land 1/4 \right)$, then the following statements hold:

\noindent (a) There is an absolute positive constant $C$ such that we have 
\begin{align} \label{eq:equiv-bound-skipped}
\EE \left( \Mhat_{\some}(\T; \ell) - M_{\pi} \right)^2 \leq C \cdot \frac{\T}{n} \land 1 \quad \text{ for all } \quad \ell = 0, \ldots, 2\T - 1.
\end{align}

\noindent (b) There is an absolute positive constant $C$ such that we have 
\begin{align} \label{eq:equiv-bound-skipped-zeta}
\EE \left( \Mhat_{\some, \leq \zeta}(\T; \ell) - M_{\pi} \right)^2 \leq C \cdot \frac{(\zeta + 1) \T}{n} \land 1 \quad \text{ for all } \quad \ell = 0, \ldots, 2\T - 1.
\end{align}
\end{proposition}
We prove part (a) of Proposition~\ref{prop:skipped-reduction} in proving Theorem~\ref{thm:main-upper-bound}; see Section~\ref{sec:main-proof}. We prove part (b) of Proposition~\ref{prop:skipped-reduction} in proving Theorem~\ref{thm:main-upper-bound-zeta}; see Section~\ref{sec:main-proof-zeta}.

\begin{figure}[ht!]
    \centering
    \includegraphics[clip, trim=12cm 1cm 12cm 1cm, width=0.49\textwidth]{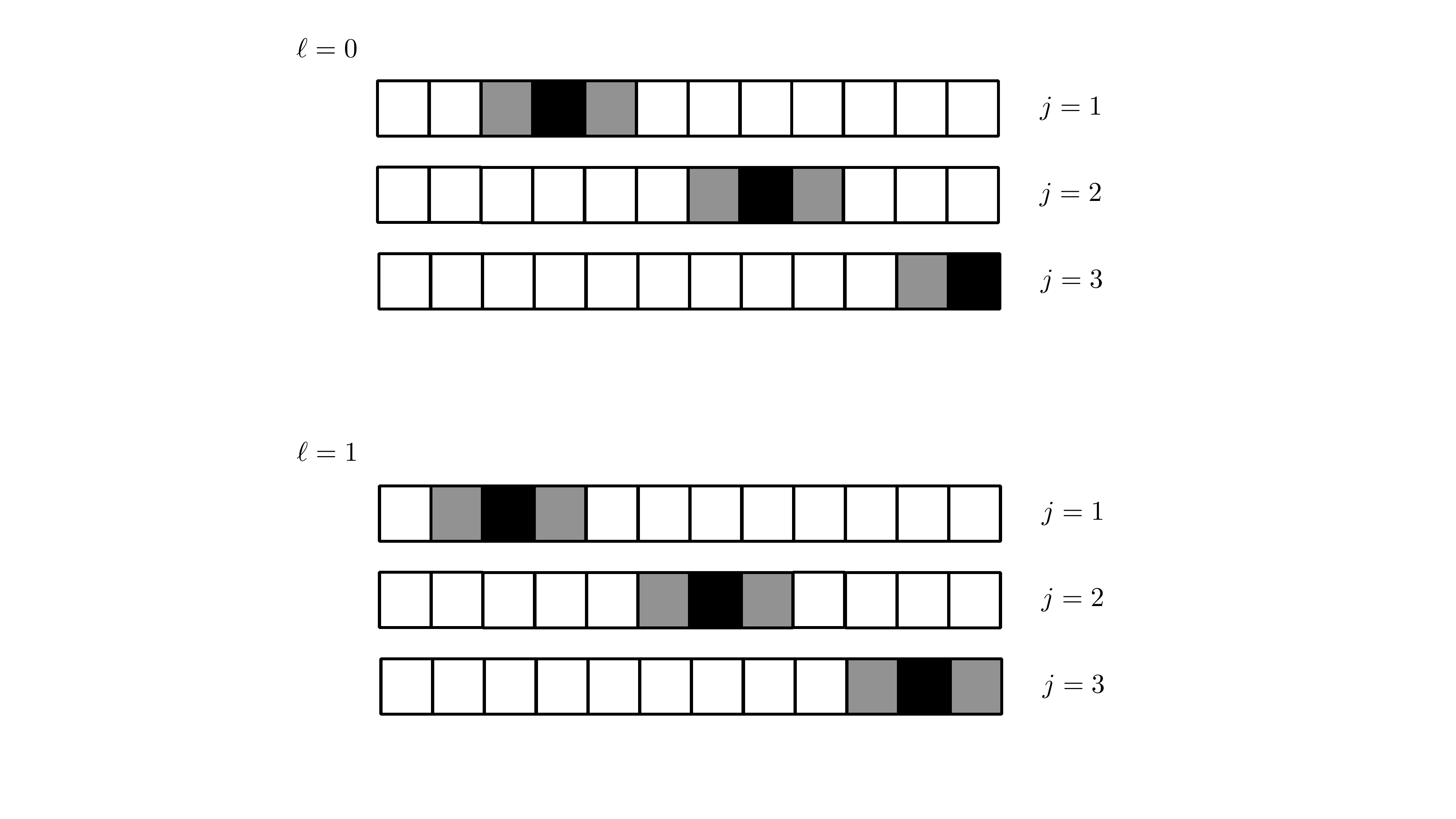}
    \includegraphics[clip, trim=12cm 1cm 12cm 1cm, width=0.49\textwidth]{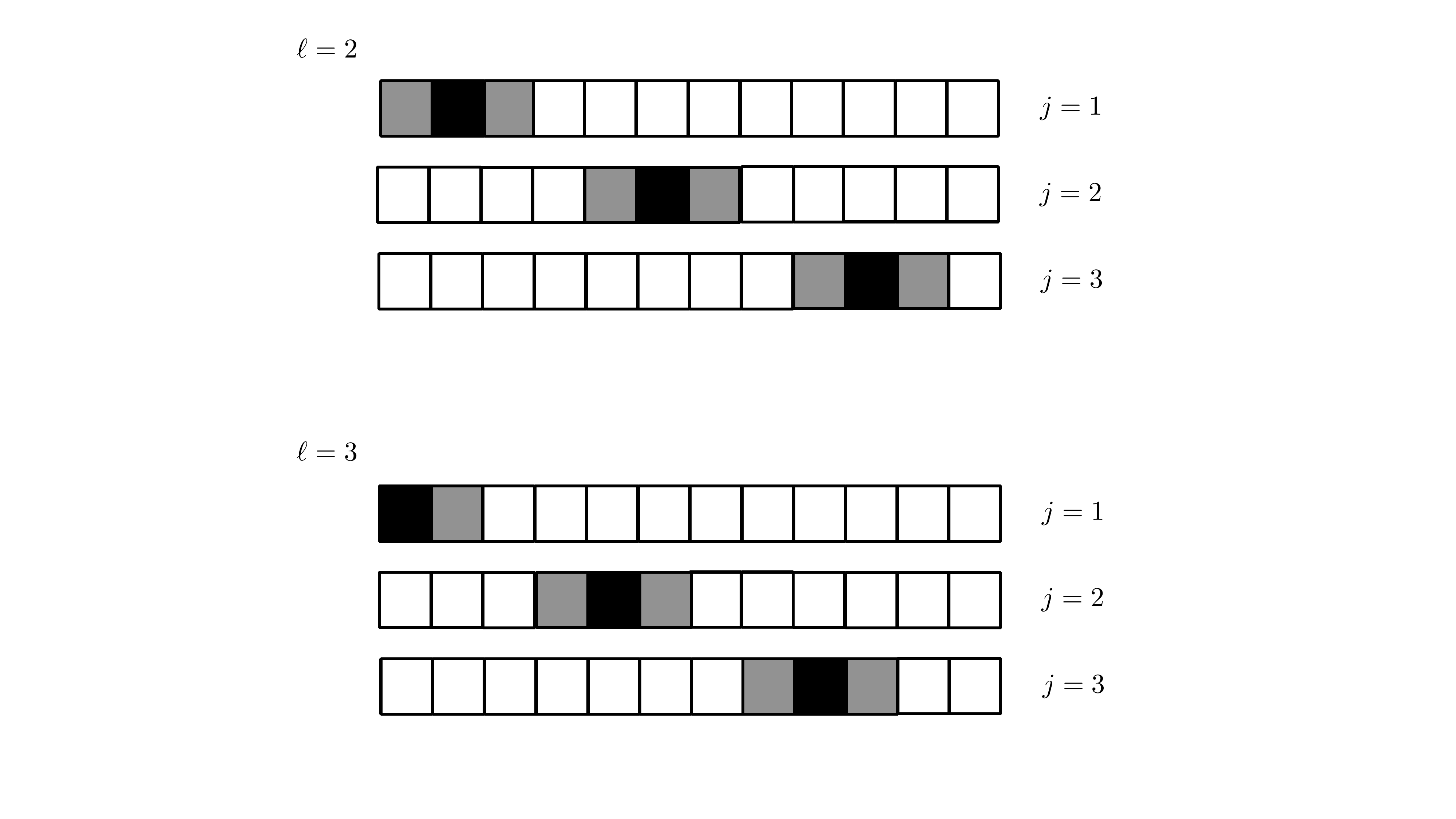}
    \caption{Schematic showing our estimator construction for $n = 12$ and $\tau = 2$, so that $n_0 = n / (2\T) = 3$. For the various values $j \in [n_0]$ and $\ell = 0, 1, \ldots, 2\T - 1$, the index $2\T j - \ell$ is shown in black and the window of size $\tau - 1$ on either size of it is excluded when computing the estimator $\Mhat_{\T}^{(2\T j - \ell)}$.  
    The indices in color form the sets $\Dsetone_{2\T j - \ell} = \Dset_{j, \ell}$, while the indices in white form the sets $\Isetone_{2\T j - \ell} = \Iset_{j, \ell}$. For each $\ell$, the sets $\{\Dsetone_{j, \ell}\}_{j \in [n_0]}$ are non-overlapping.}
    \label{fig:blocks}
\end{figure}

Recall our index sets $\{ \Dsetone_i \}_{i = 1}^n$ and $\{ \Isetone_i \}_{i = 1}^n$ from Eq.~\eqref{eq:index-sets}. For $j \in [n_0]$ and $\ell = 0, 1, \ldots, 2\T - 1$, define the sets $\Dset_{j, \ell} := \Dsetone_{2\T j - \ell}$ and $\Iset_{j, \ell} := \Isetone_{2\T j - \ell}$ as the ``dependent" and ``independent" indices, respectively, for the $j$-th window, or block, of size (at most) $2\T - 1$. Mnemonically, one should view $\Dset_{j, \ell}$ as the $j$-th \emph{block} of indices and $\Iset_{j, \ell} = [n] \setminus \Dset_{j, \ell}$ as the set of indices having a \emph{hole} at block $j$.

In the sequel, we will pay special attention to the case $\ell = 0$ and analyze the estimator $\Mhat_{\some}(\T; 0)$. Consequently, we use the shorthand $\Dset_j := \Dset_{j, 0}$ and $\Iset_{j} := \Iset_{j, 0}$.
See Figure~\ref{fig:blocks} for an illustration of our notation.

Recall that $\bm{X}_{[n]} = \{ X_1, \ldots, X_n \}$ is the set (not sequence) of all random variables. In the sequel, we use the shorthand $\EE_Y \equiv \EE_{Y \sim \pi}$ and $\EE_{Y'} \equiv \EE_{Y' \sim \pi}$, where $Y, Y'$ are drawn i.i.d. and independently of $X^n$. The notation $\EE$ (without any subscript) is reserved for an expectation taken over all the randomness in the problem.

\vspace{-2mm}
\subsection{Proof of Theorem~\ref{thm:main-upper-bound}} \label{sec:main-proof}

Owing to Eq~\eqref{eq:skip-decomp}, it suffices to establish Proposition~\ref{prop:skipped-reduction}(a). As will be clear from the proof, our argument will apply to bound the MSE of the estimator $\Mhat_{\some}(\T; \ell)$ for any $\ell = 0, \ldots, 2\T -1$, so we concentrate on establishing that for an absolute positive constant $C$ and $\T \geq \tmix\left((\Tmix/n) \land 1/4 \right)$:
\begin{align} \label{eq:equiv-bound-skipped-k0}
\EE \left( \Mhat_{\some}(\T; 0) - M_{\pi} \right)^2 \leq C \cdot \frac{\T}{n} \land 1.
\end{align}
Recall the shorthand $n_0 = n/(2\T)$ and the notation $\Dset_{j}$ and $\Iset_j$ from above.
Viewing $X^n$ as fixed for the moment and writing out our estimator $\Mhat_{\some}(\T; 0) = \frac{1}{n_0} \sum_{j = 1}^{n_0} \Mhat_{\T}^{(2\T j)}$, we have that $\frac{1}{2}| \Mhat_{\some}(\T; 0) - M_{\pi}(X^n) |^2$ is equal to
\begin{align}\label{eq:theorem1-proof-firststep}
&\frac{1}{2} \cdot \left|  \frac{1}{n_0} \sum_{j = 1}^{n_0} \ind{ X_{2\T j} \notin \bm{X}_{\Iset_{j}}} - \EE_{ \substack{Y \sim \pi \\ Y \indpt X^n }} \ind{Y \notin \bm{X}_{[n]}} \right|^2 \notag \\
&\qquad \qquad \qquad \qquad\leq \underbrace{\left|  \frac{1}{n_0} \sum_{j = 1}^{n_0} \left( \EE_{ \substack{Y \sim \pi \\ Y \indpt X^n }} \ind{Y \notin \bm{X}_{\Iset_{j}}} - \EE_{ \substack{Y \sim \pi \\ Y \indpt X^n }} \ind{Y \notin \bm{X}_{[n]} } \right) \right|^2}_{T_1} \notag \\
&\qquad \qquad \qquad \qquad \qquad \qquad + \underbrace{\left|  \frac{1}{n_0} \sum_{j = 1}^{n_0} \left( \ind{ X_{2\T j} \notin \bm{X}_{\Iset_{j}}} - \EE_{ \substack{Y \sim \pi \\ Y \indpt X^n }} \ind{Y \notin \bm{X}_{\Iset_{j}}} \right) \right|^2}_{T_2},
\end{align}
where we have added and subtracted the term $\frac{1}{n_0} \sum\limits_{j = 1}^{n_0} \mathop{\EE}\limits_{ \substack{Y \sim \pi \\ Y \indpt X^n }} \ind{Y \notin \bm{X}_{\Iset_{j}}}$ inside the expression \mbox{$| \Mhat_{\some}(\T; 0) - M_{\pi}(X^n) |$} and used the inequality $\frac{1}{2}(a+b)^2\le (a^2+b^2)$. 

We now bound $\EE[T_1]$ and $\EE[T_2]$, in turn.

\subsubsection{Bounding $\EE[T_1]$}\label{sec:T1bound}

Note that $T_1$ resembles a \emph{conditional} squared bias term. For each $j \in [n_0]$, define the random variable $P_j := \ind{Y \notin \bm{X}_{\Iset_{j}}} -  \ind{Y \notin \bm{X}_{[n]} }$.
Applying Lemma~\ref{lem:diff-ind}, we have
\[
P_j = \ind{Y \in \bm{X}_{\Dset_j}} \cdot \ind{Y \notin \bm{X}_{\Iset_j}},
\] 
which is the indicator that $Y$ appears in the sequence \emph{only in} block $\Dset_j$.

Since all blocks $\{\Dset_j \}_{j = 1}^{n_0}$ are non-overlapping, we have 
\begin{equation}
\bigsqcup_{j' \in [n_0] \setminus j} \Dset_{j'} \subset \Iset_j. \label{eq:DinH}
\end{equation}
Now suppose that for some $j$ we have $\ind{Y \in \bm{X}_{\Dset_j}} \cdot \ind{Y \notin \bm{X}_{\Iset_j}} = 1$, which implies \mbox{$\ind{Y \notin \bm{X}_{\Iset_j}} = 1$} and $\ind{Y \in \bm{X}_{\Iset_j}} = 0$. Then we must have 
\begin{align*}
\sum_{j' \in [n_0] \setminus j} \ind{Y \in \bm{X}_{\Dset_{j'}}} \cdot \ind{Y \notin \bm{X}_{\Iset_{j'}}} &\overset{\1}{\leq} \sum_{j' \in [n_0] \setminus j} \ind{Y \in \bm{X}_{\Dset_{j'}}}\\ 
&\leq \ind{Y \in \bm{X}_{\bigsqcup_{j' \in [n_0] \setminus j}\Dset_{j'}}}\\
&\overset{\2}{\leq} \ind{Y \in \bm{X}_{\Iset_{j}}} = 0,
\end{align*}
where $\1$ follows because $\ind{Y \notin \bm{X}_{\Iset_{j'}}}\le 1$ and $\2$ follows by Eq.~\eqref{eq:DinH}. Thus, 
\[\sum_{j = 1}^{n_0} P_j =
\sum_{j = 1}^{n_0} \ind{Y \in \bm{X}_{\Dset_j}} \cdot \ind{Y \notin \bm{X}_{\Iset_j}} \leq 1,\]
pointwise for every sequence $X^n$. Said another way, the term $\sum_{j = 1}^{n_0} P_j$ is equal to the indicator that $Y$ appears in exactly one block, and therefore must be at most equal to $1$.

Putting together the pieces, we have
\[
T_1 = \frac{1}{n^2_0} \left( \EE_Y \sum_{j = 1}^{n_0} P_j  \right)^2 \leq \frac{1}{n_0^2},
\]
and so $\EE[T_1] \leq \frac{1}{n_0^2} \lesssim \left(\frac{\T}{n} \right)^2 \land 1$.

\subsubsection{Bounding $\EE[T_2]$} \label{sec:T2bound}
We note that $T_2$ resembles a \emph{conditional} variance term.
Define as shorthand the random variables $Z_j := \ind{X_{2\T j} \notin \bm{X}_{\Iset_{j}}} - \EE_{Y} \ind{Y \notin \bm{X}_{\Iset_{j}}}$ for all $j \in [n_0]$.
Then we have
\begin{align*}
T_2 &= \frac{1}{n_0^2} \sum_{j,k = 1}^{n_0} Z_j Z_k \leq \frac{1}{n_0} + \frac{1}{n_0^2} \sum_{j = 1}^{n_0} \sum_{ \substack{k =1 \\ k \neq j}}^{n_0} Z_j Z_k,
\end{align*}
where the inequality follows since $Z_j \in [-1,1]$ for all $j \in [n_0]$.
Therefore, it suffices to bound the cross terms when $j \neq k$.
For each $j, k \in [n_0]$ with $j \neq k$, define the random variables
\begin{align} \label{eq:qi-bias}
Q_{j, k} =  \ind{Y \notin \bm{X}_{\Iset_j \cap \Iset_k}} - \ind{Y \notin \bm{X}_{\Iset_k}}
\end{align}
The following lemma relates the expectation of the cross terms to expectations of the random variables defined above.
\begin{lemma} \label{lem:new-lemma-cross-terms}
Suppose $\T \geq \tmix(\epsilon)$. Then for each $j \neq k$, we have
\begin{align*}
    \EE[Z_j Z_k] \leq \frac{5}{2} \EE[Q_{j, k}] + \frac{5}{2} \EE[Q_{k, j}] + 16 \epsilon,
\end{align*}
where the random variables $\{ Q_{j,k} \}$ are as defined in Eq.~\eqref{eq:qi-bias}. 
\end{lemma}
We take Lemma~\ref{lem:new-lemma-cross-terms} as given for the moment and prove it in Section~\ref{sec:pf-lemma-cross-terms}. Let us now use it to bound $\EE[T_2]$. Applying Lemma~\ref{lem:diff-ind}, we may write $Q_{j, k} =  \ind{Y \in \bm{X}_{\Iset_k \setminus \Iset_j}} \cdot \ind{Y \notin \bm{X}_{\Iset_j \cap \Iset_k}}$. 

Now consider some fixed $k \in [n_0]$. Since the sets $\{\Dset_j\}_{j = 1}^{n_0}$ are non-overlapping, we have \mbox{$\Iset_k \setminus \Iset_j = \Dset_j$}, and 
\[
\Iset_j \cap \Iset_k \supset \bigsqcup_{j' \in [n_0] \setminus \{j,k\}} \Dset_{j'}.
\]
If for some $j \neq k$, we have
\[
\ind{Y \in \bm{X}_{\Iset_k \setminus \Iset_j}} \cdot \ind{Y \notin \bm{X}_{\Iset_j \cap \Iset_k}} = 1,
\]
then
\begin{align*}
\sum_{j' \in [n_0] \setminus \{j, k\}} \ind{Y \in \bm{X}_{\Iset_k \setminus \Iset_{j'}}} \cdot \ind{Y \notin \bm{X}_{\Iset_{j'} \cap \Iset_k}}
&\leq \sum_{j' \in [n_0] \setminus \{j, k\}} \ind{Y \in \bm{X}_{\Dset_{j'}}} \\
&\leq  \ind{Y \in \bm{X}_{\Iset_{j} \cap \Iset_k}} = 0.
\end{align*}
Consequently, we have
\[
\sum_{j \in [n_0] \setminus k} Q_{j, k} \leq 1.
\]
Essentially, we have shown that $\sum_{j \in [n_0] \setminus k} Q_{j,k}$ is at most the indicator that $Y$ appears in exactly one block (other than $\Dset_k$), which is at most $1$.

Applying Lemma~\ref{lem:new-lemma-cross-terms} and using the linearity of expectation then yields
\[
\sum_{j = 1}^{n_0} \sum_{ \substack{k =1 \\ k \neq j}}^{n_0} \EE[Z_j Z_k] \leq 5 n_0 + 16n_0^2 \epsilon.
\]
Consequently, we have $\EE[T_2] \lesssim \frac{1}{n_0} + \epsilon$.
Substituting $\epsilon = \frac{\Tmix}{n}$ and noting that $\T \geq \Tmix$ by assumption, we obtain $\EE[T_2] \leq C \cdot \frac{\T}{n} \land 1$.

\medskip
Putting together our bounds on $\EE[T_1]$ and $\EE[T_2]$ establishes Theorem~\ref{thm:main-upper-bound}.
It remains to prove Lemma~\ref{lem:new-lemma-cross-terms}.

\subsubsection{Proof of Lemma~\ref{lem:new-lemma-cross-terms}} \label{sec:pf-lemma-cross-terms}

Define
\begin{align} \label{eq:Qtilde}
\widetilde{Q}_{j,k} := \EE_Y[Q_{j, k}] =  \EE_{Y} [\ind{Y \notin \bm{X}_{\Iset_j \cap \Iset_k} }  - \ind{Y \notin \bm{X}_{\Iset_k}} ] 
\end{align}
for convenience. Note that by Lemma~\ref{lem:diff-ind}, we have $\widetilde{Q}_{j,k} = \EE_Y [ \ind{Y \in \bm{X}_{\Iset_k \setminus \Iset_j} }  \cdot \ind{Y \notin \bm{X}_{\Iset_j \cap \Iset_k}}]$, so that $\widetilde{Q}_{j,k} \in [0,1]$.
We have the decomposition
\small
\begin{align*}
    & Z_j Z_k \notag \\
    &= \left( \ind{ X_{2\T j} \notin \bm{X}_{\Iset_{j}} } - \EE_Y [ \ind{ Y \notin \bm{X}_{\Iset_{j} \cap \Iset_k} } ] + \widetilde{Q}_{k, j} \right) \cdot 
    \left( \ind{ X_{2\T k} \notin \bm{X}_{\Iset_{k}} } - \EE_{Y'} [ \ind{ Y' \notin \bm{X}_{\Iset_{j} \cap \Iset_k} } ] + \widetilde{Q}_{j, k} \right) \\
    &\leq \underbrace{ \left( \ind{ X_{2\T j} \notin \bm{X}_{\Iset_{j}} } - \EE_Y [ \ind{ Y \notin \bm{X}_{\Iset_{j} \cap \Iset_k} } ]  \right) \cdot 
    \left( \ind{ X_{2\T k} \notin \bm{X}_{\Iset_{k}} } - \EE_{Y'} [ \ind{ Y' \notin \bm{X}_{\Iset_{j} \cap \Iset_k} } ]  \right) }_{U_{j, k}} \\
    &\qquad \qquad \qquad \qquad \qquad  + \widetilde{Q}_{j, k} + \widetilde{Q}_{k, j} + \widetilde{Q}_{j, k} \cdot \widetilde{Q}_{k, j} \\
    &\overset{\1}{\leq} U_{j, k} + \frac{3}{2} ( \widetilde{Q}_{j, k} + \widetilde{Q}_{k, j} ).
\end{align*}
\normalsize
Here step $\1$ follows by the following sequence of algebraic inequalities: Given that every $\widetilde{Q}_{j,k}$ is bounded in the range $[0,1]$, we have $\widetilde{Q}_{j, k} \cdot \widetilde{Q}_{k, j} \leq \sqrt{\widetilde{Q}_{j, k} \cdot \widetilde{Q}_{k, j}} \leq \frac{1}{2} ( \widetilde{Q}_{j, k} + \widetilde{Q}_{k, j})$.

It remains to establish that $\EE[U_{j, k}] \leq \EE_{X^n}[\widetilde{Q}_{j,k}] + \EE_{X^n}[\widetilde{Q}_{k, j}] + 16 \epsilon$. We have the further decomposition
\small
\begin{align}
\EE[U_{j, k}] &= \underbrace{\EE [\ind{X_{2\T j} \notin \bm{X}_{\Iset_{j}}} \cdot \ind{X_{2\T k} \notin \bm{X}_{\Iset_{k}}}]}_{U_1}
    - \underbrace{\EE_{X^n} \left[\ind{X_{2\T j} \notin \bm{X}_{\Iset_{j}}} \cdot \EE_{Y'} [\ind{Y' \notin \bm{X}_{\Iset_j \cap \Iset_{k}}}]\right]}_{U_2} \notag \\
    &\; - \underbrace{\EE_{X^n}\left[\ind{X_{2\T k} \notin \bm{X}_{\Iset_{k}}} \cdot \EE_{Y} [\ind{Y \notin \bm{X}_{\Iset_{j} \cap \Iset_k}}]\right]}_{U_3} + \underbrace{\EE_{X^n} \left[ \EE_{Y} [\ind{Y \notin \bm{X}_{\Iset_{j} \cap \Iset_k }}] \cdot \EE_{Y'} [\ind{Y' \notin \bm{X}_{\Iset_j \cap \Iset_{k}}}] \right]}_{U_4} \label{eq:four-terms}
\end{align}
\normalsize
We now bound each of the above terms in turn.

\noindent To begin, we notice that $\ind{X_{2\T j} \notin \bm{X}_{\Iset_{j}}}\le \ind{X_{2\T j} \notin \bm{X}_{\Iset_j \cap \Iset_k}}$ and $\ind{X_{2\T k} \notin \bm{X}_{\Iset_{k}}}\le \ind{X_{2\T k} \notin \bm{X}_{\Iset_j \cap \Iset_k}}$ to bound $U_1$ as
\begin{align}
U_1 &\leq \EE [\ind{X_{2\T j} \notin \bm{X}_{\Iset_j \cap \Iset_k}} \cdot \ind{X_{2\T k} \notin \bm{X}_{\Iset_j \cap \Iset_k} }] \notag \\
&\overset{\1}{\leq} \EE [\ind{Y' \notin \bm{X}_{\Iset_j \cap \Iset_k}} \cdot \ind{Y \notin \bm{X}_{\Iset_j \cap \Iset_k} }] + 8 \epsilon \notag \\
&\overset{\2}{=} \EE_{X^n} \left\{ \EE_{Y'} [\ind{Y' \notin \bm{X}_{\Iset_j \cap \Iset_k} } ] \cdot \EE_Y [\ind{Y \notin \bm{X}_{\Iset_j \cap \Iset_k} }] \right\} + 8 \epsilon. \label{eq:U1-bound}
\end{align}
Here, step $\1$ uses Lemma~\ref{lem:two-bridges} (applied with $i_1 = 2\T \min\{j, k\}, i_2 = 2\T \max\{j, k\}$, and noting that $i_2 - i_1 \geq 2\T$ for all $j \neq k$) and step $\2$ follows because $Y$ and $Y'$ are independent of everything else.

Proceeding to the next term, 
note that $U_2$ may be viewed as the expectation over $X^n$ of 
\[
f(X_{2\T j}; \bm{X}_{\Iset_j}) := \ind{X_{2\T j} \notin \bm{X}_{\Iset_{j}}} \cdot \EE_Y [\ind{Y \notin \bm{X}_{\Iset_{k} \cap \Iset_{j}}}], 
\]
which is bounded in the range $[0,1]$.
Since $\T \geq \tmix(\epsilon)$, we may apply Lemma~\ref{lem:Markov-bridge} (applied with $i = 2j\T$) to obtain
$|\EE [f(X_{2\T j}; \bm{X}_{\Iset_j})] - \EE [f(Y'; \bm{X}_{\Iset_j})] | \leq 4\epsilon$.
Thus,
\begin{align} \label{eq:U2-bound}
U_2 &\geq \EE_{X^n} \left[\EE_{Y'}[\ind{Y' \notin \bm{X}_{\Iset_{j}}}] \cdot \EE_Y [\ind{Y \notin \bm{X}_{\Iset_{j} \cap \Iset_k}}]\right] - 4\epsilon \notag \\
&\overset{\1}{\geq} \EE_{X^n} \left[\EE_{Y'}[\ind{Y' \notin \bm{X}_{\Iset_{j} \cap \Iset_k}}] \cdot \EE_Y [\ind{Y \notin \bm{X}_{\Iset_{j} \cap \Iset_k}}]\right] - \EE_{X^n}[\widetilde{Q}_{k,j}] - 4\epsilon, 
\end{align}
where step $\1$ follows because
\begin{align*}
&\EE_{Y'}[\ind{Y' \notin \bm{X}_{\Iset_{j}}}] \cdot \EE_Y [\ind{Y \notin \bm{X}_{\Iset_{j} \cap \Iset_k}}]\\
&\quad = \EE_{Y'}[\ind{Y' \notin \bm{X}_{\Iset_{j} \cap \Iset_k}}] \cdot \EE_Y [\ind{Y \notin \bm{X}_{\Iset_{j} \cap \Iset_k}}]+\EE_{Y'}[\ind{Y' \notin \bm{X}_{\Iset_{j}}}] \cdot \EE_Y [\ind{Y \notin \bm{X}_{\Iset_{j} \cap \Iset_k}}]\\
&\quad\phantom{\EE_{Y'}[\ind{Y' \notin \bm{X}_{\Iset_{j} \cap \Iset_k}}] \cdot \EE_Y [\ind{Y \notin \bm{X}_{\Iset_{j} \cap \Iset_k}}]} -\EE_{Y'}[\ind{Y' \notin \bm{X}_{\Iset_{j} \cap \Iset_k}}] \cdot \EE_Y [\ind{Y \notin \bm{X}_{\Iset_{j} \cap \Iset_k}}]\\
&\quad = \EE_{Y'}[\ind{Y' \notin \bm{X}_{\Iset_{j} \cap \Iset_k}}] \cdot \EE_Y [\ind{Y \notin \bm{X}_{\Iset_{j} \cap \Iset_k}}]-\widetilde{Q}_{k,j}\EE_Y [\ind{Y \notin \bm{X}_{\Iset_{j} \cap \Iset_k}}]\\
&\quad \ge \EE_{Y'}[\ind{Y' \notin \bm{X}_{\Iset_{j} \cap \Iset_k}}] \cdot \EE_Y [\ind{Y \notin \bm{X}_{\Iset_{j} \cap \Iset_k}}]-\widetilde{Q}_{k,j} 
\end{align*}
with the last inequality holding because of the inclusion $\EE_Y [\ind{Y \notin \bm{X}_{\Iset_{j} \cap \Iset_k}}]\in[0,1]$.

By an identical argument to the above, we have
\begin{align} \label{eq:U3-bound}
U_3 \geq \EE_{X^n} \left[\EE_{Y'}[\ind{Y' \notin \bm{X}_{\Iset_{j} \cap \Iset_k}}] \cdot \EE_Y [\ind{Y \notin \bm{X}_{\Iset_{j} \cap \Iset_k}}]\right] - \EE_{X^n}[\widetilde{Q}_{j,k}] - 4\epsilon.
\end{align}

Putting Eqs.~\eqref{eq:U1-bound},~\eqref{eq:U2-bound} and~\eqref{eq:U3-bound} together with the definition of $U_4$ and performing the requisite cancellations, we have
\begin{align*}
\EE[U_{j, k}] = U_1 - U_2 - U_3 + U_4 \leq \EE_{X^n}[\widetilde{Q}_{j,k}] + \EE_{X^n}[\widetilde{Q}_{k,j}] + 16 \epsilon,
\end{align*}
as claimed.
\qed

\subsection{Proof of Theorem~\ref{thm:variance-bound-MM}} \label{sec:var-proof}

Since the best constant predictor of a random variable is its expectation, we have
\begin{align} \label{eq:three-term-decomp}
\Var(M_{\pi}(X^n)) &\leq \EE (M_{\pi}(X^n) - \EE \Mhat_{\some}(\T))^2 \notag \\
& = \EE \left((M_{\pi}(X^n) - \Mhat_{\some}(\T)) + (\Mhat_{\some}(\T) - \EE \Mhat_{\some}(\T))\right)^2 \notag \\
&\overset{\1}{\leq} 2 \MSE(\Mhat_{\some}(\T), M_{\pi}) + 2 \Var(\Mhat_{\some}(\T)),
\end{align}
where step $\1$ follows by using $(a+b)^2\le 2a^2+2b^2$.
Throughout the rest of this proof, we will choose $\T = \tmix\left( (\Tmix/n) \land 1/4 \right) \lesssim \Tmix \cdot \log (1 + n/\Tmix)$ (see Eq.~\eqref{eq:tmix-logs}). 

\paragraph{Bounding MSE:} Applying Theorem~\ref{thm:main-upper-bound} yields the direct bound 
\[
\MSE(\Mhat_{\some}(\T), M_{\pi}(X^n)) \lesssim \T/n \lesssim \frac{\Tmix}{n} \cdot \log (1 + n/\Tmix).
\]
It remains to bound the variance term on the RHS of Ineq.~\eqref{eq:three-term-decomp}.

\paragraph{Bounding variance of estimator:}
To bound the variance, we make use of the fact that the estimator $\Mhat_{\some}(\T)$ satisfies a bounded differences property~\citep{doob1940regularity,mcdiarmid1989method} with respect to the variables $(X_1, \ldots, X_n)$. This allows us to obtain a sub-Gaussian concentration inequality, which in turn is used to bound variance. We state this result as a lemma that may be of independent interest.

To set up some notation, let $X^{-i} = (X_1, \ldots, X_{i - 1}, X_{i + 1}, \ldots, X_n)$ denote the sequence without its $i$-th entry, and let $X^{(i)}(x) = (X_1, \ldots, X_{i - 1}, x, X_{i + 1}, \ldots, X_n)$ denote the sequence with $x$ at the $i$-th position and $X^{-i}$ in the remaining positions. The maximum difference witnessed by a function $f: \Xspace^n \to \reals$ at its $i$-th index is given by
\begin{align*}
\delta_i(f) := \max_{X^{-i} \in \Xspace^{n-1}} \; \delta_i(f; X^{-i}), \qquad \text{where} \qquad 
\delta_i(f; X^{-i}) := \max_{x, x' \in \Xspace} |f(X^{(i)}(x)) - f(X^{(i)}(x'))|.
\end{align*}
For a positive (nonrandom) scalar $b_i$, the function $f$ is said to satisfy a $b_i$ bounded differences inequality at index $i$ if $\delta_i(f) \leq b_i$.
\begin{lemma} \label{lem:bdd-diff}
Define the shorthand $\Mhat_{\some}^{\T}:= \Mhat_{\some}(\T)$ for convenience. For every $\T \in [n]$, the function $x^n \mapsto \Mhat_{\some}^{\T}(x^n)$ satisfies a $\frac{4\T}{n}$ bounded differences property on all indices, in that $\max_{i \in [n]} \delta_i(\Mhat_{\some}^{\T}) \leq \frac{4\T}{n}$.
\end{lemma}
We prove Lemma~\ref{lem:bdd-diff} shortly.
For the moment, applying it in conjunction with Corollary 2.10 and Remark 2.11 of~\citet{paulin2015concentration}, we obtain that for all $t \geq 0$,
\begin{align*}
\Pr\left \{ |\Mhat_{\some}(\T) - \EE[\Mhat_{\some}(\T)]| \geq t \right\} \leq 2 \exp \left( -c \cdot \frac{n t^2}{\Tmix \T} \right)
\end{align*}
for some universal constant $c > 0$.
Integrating the tail bound and using $\EE[Z]=\int_0^{\infty}\Pr(Z\ge z)dz$ for the non-negative random variable $Z=|\Mhat_{\some}(\T) - \EE[\Mhat_{\some}(\T)]|^2$ yields that for any $\T \in [n]$, we have
$\Var(\Mhat_{\some}(\T)) \lesssim \frac{\T \cdot \Tmix}{n}$.

\medskip

Using Ineq.~\eqref{eq:three-term-decomp}, setting $\T \asymp \Tmix \log(1 + n/\Tmix)$, and putting together the bounds on MSE and variance yields
\[
\Var(M_{\pi}(X^n)) \leq C \cdot \frac{\Tmix^2}{n} \cdot \log (1 + n/\Tmix) \land 1,
\]
as desired.
\qed

\subsubsection{Proof of Lemma~\ref{lem:bdd-diff}}

We will show that $\delta_i(\Mhat^{\T}_{\some};X^{-i}) \leq \frac{4\T}{n}$ for a fixed $i \in [n]$ and any $X^{-i} \in \Xspace^{n-1}$, which directly implies the desired result.
Consider the sequences $X^{(i)}(x)$ and $X^{(i)}(x')$.
Recall from Eq.~\eqref{eq:estimator-some} that we defined 
$\Mhat^{\T}_{\some}(X^n) = \frac{1}{n} \sum_{i'=1}^n \Mhat_{\T}^{(i')}(X^n)$, where $\Mhat_{\T}^{(i')}(X^n) = \ind{X_{i'} \notin \bm{X}_{\Isetone_{i'}}}$.
From this, applying the triangle inequality yields
\begin{align*}
    \delta_i(\Mhat^{\T}_{\some};X^{-i}) \leq \max_{x, x' \in \Xspace} \;  \frac{\sum_{i'=1}^n \ind{\Mhat_{\T}^{(i')}(X^{(i)}(x)) \neq \Mhat_{\T}^{(i')}(X^{(i)}(x'))}}{n}.
\end{align*}
Thus, it suffices to show that for any $x, x' \in \Xspace$, the total number of indices $i' \in [n]$ for which $\Mhat_{\T}^{(i')}$ changes when we switch $X_i$ from $x \to x'$, i.e. the quantity
\[
V(x, x') := \sum_{i'=1}^n \ind{\Mhat_{\T}^{(i')}(X^{(i)}(x)) \neq \Mhat_{\T}^{(i')}(X^{(i)}(x'))}, 
\]
is less than or equal to $4\tau$.

Let $Z_{i'}=\ind{\Mhat_{\T}^{(i')}(X^{(i)}(x)) \neq \Mhat_{\T}^{(i')}(X^{(i)}(x'))}$. On the one hand, if $X_{i'}\notin\{x,x'\}$, then $Z_{i'}=0$ because changing $X_i$ from $x$ to $x'$ has no impact on the indicator $\ind{X_{i'} \notin \bm{X}_{\Isetone_{i'}}}$. 
On the other hand,
if there are at least $2\T+1$ occurrences of $x$ in the sequence $X^n$, and if $X_{i'}=x$, then $\ind{X_{i'} \notin \bm{X}_{\Isetone_{i'}}}=0$ for any $i'$ because $\Dsetone_{i'}$ has at most $2\T - 1$ elements and some $x$ remains in $\bm{X}_{\Isetone_{i'}}$. So, $Z_{i'}=0$. An analogous argument holds if there are at least $2\T+1$ occurrences of $x'$ in the sequence $X^n$.

Thus, the number of indices $i'$ for which $Z_{i'}=1$ can be at most $2\T+2\T=4\T$. This proves the bound $V(x, x') \le 4\T$ for any pair $(x, x')$, as claimed.
\qed

\begin{remark}\label{rem:bdd-diff-tight}
The upper bound in Lemma~\ref{lem:bdd-diff} is tight up to a factor $4$.
To see this, fix $x \neq x' \in \Xspace$, the index $i = 2\T + 1$ and form the sequence $X^n = (X_1, \ldots, X_n)$ by setting
\begin{align*}
    X_{i'} =
    \begin{cases}
        x \; &\text{ if } i' \in \{1,\ldots,\T\} \\
        x' &\text{ otherwise. } 
    \end{cases}
\end{align*}
Then, for all the indices $i' \in \{1,\ldots,\T\}$ we have $\Mhat_{\T}^{(i')}(X^{(i)}(x)) = 0$ but $\Mhat_{\T}^{(i')}(X^{(i)}(x')) = 1$.
Moreover, for all other indices $i' > \T$ we have $\Mhat_{\T}^{(i')}(X^{(i)}(x)) = \Mhat_{\T}^{(i')}(X^{(i)}(x')) = 0$.
This directly implies that
\begin{align*}
    \delta_i(\Mhat_{\some}(\T);X^{-i}) =  \left | \frac{\T}{n} - 0 \right| = \frac{\T}{n}.
\end{align*}
\end{remark}

\subsection{Proof of Theorem~\ref{thm:main-upper-bound-zeta}}\label{sec:main-proof-zeta}

The structure of this proof parallels the proof of Theorem~\ref{thm:main-upper-bound}; the reader is advised to read that proof first. 
It is also useful to recall the notation for index sets that was defined in Section~\ref{sec:pf-prelim}.

Owing to Eq~\eqref{eq:skip-decomp-zeta}, it suffices to establish Proposition~\ref{prop:skipped-reduction}(b). As in the case before, our argument will apply to bound the MSE of the estimator $\Mhat_{\some, \leq \zeta}(\T; \ell)$ for any $\ell = 0, \ldots, 2\T -1$, so we concentrate on establishing that for an absolute positive constant $C$ and $\T \geq \tmix\left((\Tmix/n) \land 1/4 \right)$:
\begin{align} \label{eq:equiv-bound-skipped-k0-zeta}
\EE \left( \Mhat_{\some, \leq\zeta}(\T; 0) - M_{\pi} \right)^2 \leq C \cdot \frac{(\zeta + 1) \T}{n} \land 1.
\end{align}
We next proceed via a series of steps that resembles the proof of Theorem~\ref{thm:main-upper-bound}.
Recall the notation $N_x(\bm{X}_P)$ (and $N_x(X_P)$) that we defined in Section~\ref{sec:extension}, denoting the number of occurrences of $x$ in the subset $\bm{X}_P$ (and subsequence $X_P$).
Viewing $X^n$ as fixed for the moment and writing out our estimator $\Mhat_{\some,\leq \zeta}(\tau; 0) := \frac{1}{n_0} \sum_{j=1}^{n_0} \Mhat_{\tau, \leq \zeta}^{(2\tau j)}$, a parallel argument to Eq.~\eqref{eq:theorem1-proof-firststep} yields that $\frac{1}{2} | \Mhat_{\some, \leq \zeta}(\tau;0) - M_{\pi,\leq \zeta}(X^n) |^2$ is equal to
\begin{align*}
    &\frac{1}{2} \cdot \left|  \frac{1}{n_0} \sum_{j = 1}^{n_0} \ind{ N_{X_{2\T j}}(\bm{X}_{\Iset_j}) \leq \zeta} - \EE_{ \substack{Y \sim \pi \\ Y \indpt X^n }} \ind{N_{Y}(\bm{X}_{[n]}) \leq \zeta} \right|^2 \notag \\
&\qquad \qquad \qquad \qquad\leq \underbrace{\left|  \frac{1}{n_0} \sum_{j = 1}^{n_0} \left( \EE_{ \substack{Y \sim \pi \\ Y \indpt X^n }} \ind{N_{Y}(\bm{X}_{\Iset_j}) \leq \zeta} - \EE_{ \substack{Y \sim \pi \\ Y \indpt X^n }} \ind{N_{Y}(\bm{X}_{[n]}) \leq \zeta} \right) \right|^2}_{T_1'} \notag \\
&\qquad \qquad \qquad \qquad \qquad \qquad + \underbrace{\left|  \frac{1}{n_0} \sum_{j = 1}^{n_0} \left( \ind{N_{X_{2\T j}}(\bm{X}_{\Iset_j}) \leq \zeta} - \EE_{ \substack{Y \sim \pi \\ Y \indpt X^n }} \ind{N_{Y}(\bm{X}_{\Iset_j}) \leq \zeta} \right) \right|^2}_{T_2'}.
\end{align*}
As in the proof of Theorem~\ref{thm:main-upper-bound}, we upper bound $\EE[T_1']$ and $\EE[T_2']$.

\subsubsection{Bounding $\EE[T_1']$} 

As in the proof of Theorem~\ref{thm:main-upper-bound}, $T_1'$ resembles a conditional squared bias term.
For each $j \in [n_0]$, we now define the random variable $P'_j := \ind{N_Y(\bm{X}_{\Iset_j}) \leq \zeta} - \ind{N_Y(\bm{X}_{[n]}) \leq \zeta}$.
It is easy to see that $T_1' = \frac{1}{n_0^2} \left(\EE_{\substack{Y}} \sum_{j=1}^{n_0} P'_j \right)^2$.
We will bound the term $\sum_{j=1}^{n_0} P'_j$ pointwise.
Applying Lemma~\ref{lem:diff-ind-count}, we have $P'_j \leq \ind{Y \in \bm{X}_{\Dset_j}} \cdot \ind{N_Y(\bm{X}_{\Iset_j}) \leq \zeta}$.
As also argued in Section~\ref{sec:T1bound}, since the blocks $\{\Dset_j\}_{j =1}^{n_0}$ are non-overlapping, we have 
\[
\bigsqcup_{j' \in [n_0] \setminus j} \Dset_{j'} \subset \Iset_j. 
\]
Now, suppose that for some $j$ we have $\ind{Y \in \bm{X}_{\Dset_j}} \cdot \ind{N_Y(\bm{X}_{\Iset_j}) \leq \zeta}=1$.
This means that $Y$ occurs at least once in $\Dset_j$, but its number of occurrences outside of $\Dset_j$ is at most $\zeta$.
Then, we must have
\[
\sum_{j' \in [n_0] \setminus j} \ind{Y \in \bm{X}_{\Dset_{j'}}} \cdot \ind{N_Y(\bm{X}_{\Iset_{j'}}) \leq \zeta} \leq \sum_{j' \in [n_0] \setminus j} \ind{Y \in \bm{X}_{\Dset_{j'}}} \leq N_Y(\bm{X}_{\Iset_j}) \leq \zeta.
\]
Putting these together yields $\sum_{j = 1}^{n_0} P'_j \leq \zeta + 1$ pointwise. 
Ultimately, this yields
\begin{align*}
    T'_1 = \frac{1}{n_0^2} \left(\EE_{\substack{Y}} \sum_{j=1}^{n_0} P'_j \right)^2 \leq \left(\frac{\zeta+1}{n_0}\right)^2,
\end{align*}
and so $\EE[T_1] \lesssim \left(\frac{(\zeta + 1) \T}{n}\right)^2 \land 1$.

\subsubsection{Bounding $\EE[T_2']$}

As in the proof of Theorem~\ref{thm:main-upper-bound}, $T'_2$ resembles a conditional variance term.
We now define, for all $j \in [n_0]$,  the random variables 
\[
Z'_j := \ind{N_{X_{2j \T}}(\bm{X}_{\Iset_j}) \leq \zeta} - \EE_{\substack{Y \sim \pi \\ Y \indpt X^n}} \ind{N_{Y}(\bm{X}_{\Iset_j}) \leq \zeta}.
\]
Then, we have
\begin{align*}
    T'_2 = \frac{1}{n_0^2} \sum_{j,k = 1}^{n_0} Z'_j Z'_k \leq \frac{1}{n_0} + \frac{1}{n_0^2} \sum_{j=1}^{n_0} \sum_{\substack{k=1 \\k \neq j}}^{n_0} Z'_j Z'_k,
\end{align*}
where the inequality follows since $Z'_j \in [-1,1]$ for all $j \in [n_0]$.
Therefore, it suffices to bound the cross terms when $j \neq k$.
For each $j, k \in [n_0]$ with $j \neq k$, define the random variables
\begin{align}\label{eq:qprime-bias}
Q'_{j,k} = \ind{N_Y(\bm{X}_{\Iset_j \cap \Iset_k}) \leq \zeta} - \ind{N_Y(\bm{X}_{\Iset_k}) \leq \zeta}.
\end{align}
The following lemma, which is analogous to Lemma~\ref{lem:new-lemma-cross-terms}, relates the expectation of the cross terms to the expectations of these random variables.
\begin{lemma}\label{lem:new-lemma-cross-terms-smallcount}
    Suppose $\T \geq \tmix(\epsilon)$. Then, for each $j \neq k$, we have
    \[
    \EE[Z'_j Z'_k] \leq \frac{5}{2} \EE[Q'_{j,k}] + \frac{5}{2} \EE[Q'_{k,j}] + 16\epsilon,
    \]
    where the random variables $\{Q'_{j,k}\}$ are defined as in Eq.~\eqref{eq:qprime-bias}.
\end{lemma}
We take Lemma~\ref{lem:new-lemma-cross-terms-smallcount} as given for the moment and prove it in Section~\ref{sec:pf-lemma-cross-terms-smallcount}.
We now use it to bound $\EE[T'_2]$.
Applying Lemma~\ref{lem:diff-ind-count}, we have $Q'_{j,k} \leq \ind{Y \in \bm{X}_{\Iset_k \setminus \Iset_j}} \cdot \ind{N_Y(\bm{X}_{\Iset_j \cap \Iset_k}) \leq \zeta}$.
Now, consider some fixed $k \in [n_0]$.
As described in Section~\ref{sec:T2bound}, since the sets $\{\Dset_j\}_{j=1}^{n_0}$ are non-overlapping, we have $\Iset_k \setminus \Iset_j = \Dset_j$, and 
\[
\Iset_j \cap \Iset_k \supset \bigsqcup_{j' \in [n_0] \setminus \{j,k\}} \Dset_{j'}.
\]
If for some $j \neq k$, we have
\[
\ind{Y \in \bm{X}_{\Iset_k \setminus \Iset_j}} \cdot \ind{N_Y(\bm{X}_{\Iset_j \cap \Iset_k}) \leq \zeta} = 1,
\]
it means that $Y$ occurs at least once in the block $\Dset_j$, but at most $\zeta$ times in the set $\Iset_j \cap \Iset_k$.
This implies that
\begin{align*}
    \sum_{j' \in [n_0] \setminus \{j,k\}} \ind{Y \in \bm{X}_{\Iset_k \setminus \Iset_{j'}}} \cdot \ind{N_Y(\bm{X}_{\Iset_{j'} \cap \Iset_k}) \leq \zeta} &\leq \sum_{j' \in [n_0] \setminus \{j,k\}} \ind{Y \in \bm{X}_{\Dset_{j'}}} \\
    &\leq N_Y(\bm{X}_{\Iset_j \cap \Iset_k}) \leq \zeta.
\end{align*}
Consequently, we have
\[
\sum_{j \in [n_0] \setminus k} Q'_{j,k} \leq \zeta + 1.
\]
Applying Lemma~\ref{lem:new-lemma-cross-terms-smallcount} and using the linearity of expectation then yields 
\[
\sum_{j=1}^{n_0} \sum_{\substack{k = 1 \\ k \neq j}}^{n_0} \EE[Z'_j Z'_k] \leq 5 (\zeta + 1) n_0 + 16 n_0^2 \epsilon.
\]
Consequently, we have $\EE[T'_2] \lesssim \frac{\zeta + 1}{n_0} + \epsilon$.
Substituting $\epsilon = \frac{\Tmix}{n}$ and noting that $\T \geq \Tmix$ by assumption, we obtain $\EE[T'_2] \leq C \cdot \frac{\T(\zeta + 1)}{n} \land 1$.

\medskip 

Combining our bounds on $\EE[T'_1]$ and $\EE[T'_2]$ completes the proof of Theorem~\ref{thm:main-upper-bound-zeta}.
It remains to prove Lemma~\ref{lem:new-lemma-cross-terms-smallcount}.

\subsubsection{Proof of Lemma~\ref{lem:new-lemma-cross-terms-smallcount}}\label{sec:pf-lemma-cross-terms-smallcount}
The structure of this proof closely resembles the proof of Lemma~\ref{lem:new-lemma-cross-terms}.
Define
\begin{align}
    \overline{Q}_{j,k} := \EE_{\substack{Y}}[Q'_{j,k}] = \EE_{\substack{Y}} \left[\ind{N_Y(\bm{X}_{\Iset_j \cap \Iset_k}) \leq \zeta} - \ind{N_Y(\bm{X}_{\Iset_k}) \leq \zeta}\right]
\end{align}
for convenience.
Note that by Lemma~\ref{lem:diff-ind-count}, we have 
\[
\overline{Q}_{j,k} \leq \EE_Y \left[\ind{Y \in \bm{X}_{\Iset_k \setminus \Iset_j}} \cdot \ind{N_Y(\bm{X}_{\Iset_j \cap \Iset_k}) \leq \zeta}\right] \leq 1,
\]
and moreover $Q'_{j,k} \geq 0$ pointwise so $\overline{Q}_{j,k} \geq 0$.
Therefore, $\overline{Q}_{j,k} \in [0,1]$.
We then have the decomposition
\small
\begin{align*}
    & Z'_j Z'_k \notag \\
    &= \left( \ind{ N_{X_{2\T j}}(\bm{X}_{\Iset_{j}}) \leq \zeta } - \EE_Y [ \ind{ N_Y(\bm{X}_{\Iset_{j} \cap \Iset_k}) \leq \zeta } ] + \overline{Q}_{k, j} \right) \cdot \\
    &\qquad\qquad\qquad\qquad\left( \ind{ N_{X_{2\T k}}(\bm{X}_{\Iset_{k}}) \leq \zeta } - \EE_{Y'} [ \ind{ N_{Y'} (\bm{X}_{\Iset_{j} \cap \Iset_k}) \leq \zeta } ] + \overline{Q}_{j, k} \right) \\
    &\leq \underbrace{ \left( \ind{ N_{X_{2\T j}}(\bm{X}_{\Iset_{j}}) \leq \zeta } - \EE_Y [ \ind{ N_Y(\bm{X}_{\Iset_{j} \cap \Iset_k}) \leq \zeta } ]  \right) \cdot 
    \left( \ind{ N_{X_{2\T k}}(\bm{X}_{\Iset_{k}}) \leq \zeta } - \EE_{Y'} [ \ind{ N_{Y'} (\bm{X}_{\Iset_{j} \cap \Iset_k}) \leq \zeta } ]  \right) }_{U'_{j, k}} \\
    &\qquad \qquad \qquad \qquad \qquad  + \overline{Q}_{j, k} + \overline{Q}_{k, j} + \overline{Q}_{j, k} \cdot \overline{Q}_{k, j} \\
    &\overset{\1}{\leq} U'_{j, k} + \frac{3}{2} ( \overline{Q}_{j, k} + \overline{Q}_{k, j} ).
\end{align*}
\normalsize
Here step $\1$ follows due to the following algebraic inequalities: Since each $\overline{Q}_{j,k} \in [0,1]$, we have $\overline{Q}_{j, k} \cdot \overline{Q}_{k, j} \leq \sqrt{\overline{Q}_{j, k} \cdot \overline{Q}_{k, j}} \leq \frac{1}{2} ( \overline{Q}_{j, k} + \overline{Q}_{k, j})$.

It remains to establish that $\EE[U'_{j, k}] \leq \EE_{X^n}[\overline{Q}_{j,k}] + \EE_{X^n}[\overline{Q}_{k, j}] + 16 \epsilon$. We have the further decomposition
\small
\begin{align}
&\EE[U'_{j, k}] \notag \\
&= \underbrace{\EE [\ind{N_{X_{2\T j}}(\bm{X}_{\Iset_{j}}) \leq \zeta} \cdot \ind{N_{X_{2\T k}} (\bm{X}_{\Iset_{k}}) \leq \zeta}]}_{U'_1}
    - \underbrace{\EE_{X^n} \left[\ind{N_{X_{2\T j}}(\bm{X}_{\Iset_{j}}) \leq \zeta} \cdot \EE_{Y'} [\ind{N_{Y'}(\bm{X}_{\Iset_j \cap \Iset_{k}}) \leq \zeta}]\right]}_{U'_2} \notag \\
    &\; - \underbrace{\EE_{X^n}\left[\ind{N_{X_{2\T k}}(\bm{X}_{\Iset_{k}}) \leq \zeta} \cdot \EE_{Y} [\ind{N_Y(\bm{X}_{\Iset_{j} \cap \Iset_k}) \leq \zeta}]\right]}_{U'_3} \notag \\
    &+ \underbrace{\EE_{X^n} \left[ \EE_{Y} [\ind{N_Y(\bm{X}_{\Iset_{j} \cap \Iset_k }) \leq \zeta}] \cdot \EE_{Y'} [\ind{N_{Y'}(\bm{X}_{\Iset_j \cap \Iset_{k}}) \leq \zeta}] \right]}_{U'_4} \label{eq:four-terms-prime}
\end{align}
\normalsize
We now bound each of the above terms in turn.

First, we bound $U'_1$ as
\begin{align}
U'_1 &\leq \EE [\ind{N_{X_{2\T j}}(\bm{X}_{\Iset_j \cap \Iset_k}) \leq \zeta} \cdot \ind{N_{X_{2\T k}} (\bm{X}_{\Iset_j \cap \Iset_k}) \leq \zeta }] \notag \\
&\overset{\1}{\leq} \EE [\ind{N_{Y'} (\bm{X}_{\Iset_j \cap \Iset_k}) \leq \zeta} \cdot \ind{N_Y(\bm{X}_{\Iset_j \cap \Iset_k}) \leq \zeta }] + 8 \epsilon \notag \\
&\overset{\2}{=} \EE_{X^n} \left\{ \EE_{Y'} [\ind{N_{Y'}(\bm{X}_{\Iset_j \cap \Iset_k}) \leq \zeta } ] \cdot \EE_Y [\ind{N_Y(\bm{X}_{\Iset_j \cap \Iset_k}) \leq \zeta }] \right\} + 8 \epsilon, \label{eq:U1prime-bound}
\end{align}
where step $\1$ uses Lemma~\ref{lem:two-bridges} (applied with $i_1 = 2\T\min\{j,k\}, i_2 = 2\T\max\{j,k\}$, and noting that $i_2 - i_1 \geq 2\T$ as $j \neq k$), and step $\2$ follows because $Y$ and $Y'$ are independent of everything else.

Proceeding to the next term, 
note that $U'_2$ may be viewed as the expectation over $X^n$ of 
\[
f'(X_{2j\T}; \bm{X}_{\Iset_j}) := \ind{N_{X_{2\T j}}(\bm{X}_{\Iset_{j}}) \leq \zeta} \cdot \EE_Y [\ind{N_Y(\bm{X}_{\Iset_{k} \cap \Iset_{j}}) \leq \zeta}], 
\]
which is bounded in the range $[0,1]$.
Since $\T \geq \tmix(\epsilon)$, we may now apply Lemma~\ref{lem:Markov-bridge} (for the choice $i = 2\T j$) to obtain
$
|\EE [f'(X_{2\T j}; \bm{X}_{\Iset_j})] - \EE [f'(Y'; \bm{X}_{\Iset_j})] | \leq 4\epsilon$.
Thus,
\begin{align} \label{eq:U2prime-bound}
U'_2 \notag &\geq \EE_{X^n} \left[\EE_{Y'}[\ind{N_{Y'}(\bm{X}_{\Iset_{j}}) \leq \zeta}] \cdot \EE_Y [\ind{N_Y(\bm{X}_{\Iset_{j} \cap \Iset_k}) \leq \zeta}]\right] - 4\epsilon \notag \\
&\geq \EE_{X^n} \left[\EE_{Y'}[\ind{N_{Y'}(\bm{X}_{\Iset_{j} \cap \Iset_k}) \leq \zeta}] \cdot \EE_Y [\ind{N_Y(\bm{X}_{\Iset_{j} \cap \Iset_k}) \leq \zeta}]\right] - \EE_{X^n}[\overline{Q}_{k,j}] - 4\epsilon.
\end{align}

By an identical argument to the above, we have
\begin{align} \label{eq:U3prime-bound}
U'_3 \geq \EE_{X^n} \left[\EE_{Y'}[\ind{N_{Y'}(\bm{X}_{\Iset_{j} \cap \Iset_k}) \leq \zeta}] \cdot \EE_Y [\ind{N_Y(\bm{X}_{\Iset_{j} \cap \Iset_k}) \leq \zeta}]\right] - \EE_{X^n}[\overline{Q}_{j,k}] - 4\epsilon.
\end{align}

Putting Eqs.~\eqref{eq:U1prime-bound},~\eqref{eq:U2prime-bound} and~\eqref{eq:U3prime-bound} together with the definition of $U'_4$ and performing the requisite cancellations, we have
\begin{align*}
\EE[U'_{j, k}] = U'_1 - U'_2 - U'_3 + U'_4 \leq \EE_{X^n}[\overline{Q}_{j,k}] + \EE_{X^n}[\overline{Q}_{k,j}] + 16 \epsilon.
\end{align*}
This completes the proof of the lemma and so the proof of Theorem~\ref{thm:main-upper-bound-zeta}.
\qed

\section{Discussion} \label{sec:discussion}

We presented the \textsc{WingIt} estimator for estimating the stationary mass missing from a Markovian sequence. While the vanilla Good--Turing estimator can suffer constant bias in the Markovian setting, our estimator achieves (near) minimax optimal mean-squared error over mixing Markov chains. It can also be computed with a linear-time algorithm, and performs favorably in our experiments, even in language text applications in which the Markovian assumption is clearly violated. We also presented a variant of \textsc{WingIt} for estimating the small-count probability in a Markov sequence and established mean squared error bounds for this task.

Our work leaves open several important and intriguing questions aside from the conjectured improvement of Theorem~\ref{thm:variance-bound-MM}.
First, while Theorem~\ref{thm:main-upper-bound} provides a complete picture---up to a  logarithmic factor---for stationary missing mass estimation from the point of view of MSE, it would be interesting to complement this result with a concentration inequality.
Such a concentration result could, for instance, be used to provide a provable guarantee on the validation procedure that we outlined in Section~\ref{sec:expts}.
Second, we reiterate that our estimator is only optimal up to a logarithmic factor in $n/ \Tmix$, and removing this factor to match the minimax lower bound---possibly by designing an alternative estimator---is an interesting open problem.

Third, we believe that the Markov property may not be central to our main results, and that Theorem~\ref{thm:main-upper-bound} could be extended to more general $\alpha$-mixing sequences~\citep{rosenblatt1956central}. This extension would capture, for instance, other classes of interesting temporal processes such as some hidden Markov models.
Fourth, a related point is that the assumption~\eqref{eq:tmix} of geometric ergodicity itself is central to the design and analysis of our estimator; designing estimators that do not require ergodicity---perhaps just irreducibility~\citep{fried2023alpha}---would be of great interest and likely require new ideas. 

Finally, it would be interesting to estimate other functionals of the Markov chain other than the stationary missing mass and solve related estimation problems such as competitive distribution estimation of the stationary measure.
Our extensions to estimating the mass of elements occurring at most $\zeta$ times in Section~\ref{sec:extension} might be a useful starting point as in the i.i.d. case~\citep{drukh2005concentration,acharya2013optimal}, but several questions remain, such as obtaining a bound on the error of estimating all such quantities uniformly over $\zeta \in \{0, 1, \ldots, n \}$.

\subsection*{Acknowledgments}
This work was supported in part by National Science Foundation grants CCF-2107455,  DMS-2210734, CCF-2239151 and IIS-2212182, and by research awards/gifts from Adobe, Amazon, Google and MathWorks. AP thanks Wenlong Mou for helpful discussions. We are also thankful to the anonymous reviewers, whose comments improved the scope and presentation of the manuscript.

\bibliographystyle{abbrvnat}
\bibliography{refs-revision}

\appendix
\section{Technical lemmas} \label{sec:lemmas}

In this section, we collect technical lemmas that were stated and used in the main paper.
We first collect lemmas that were used to formalize basic calculations for the Good--Turing estimator, and next lemmas that were used in the proofs of the main results (Theorems~\ref{thm:main-upper-bound} and~\ref{thm:variance-bound-MM}).

\subsection{Elementary lemmas}
Our first lemma shows a tight characterization of the mixing time $\Tmix = \tmix(1/4)$ for the class of sticky Markov chains, defined in Eq.~\eqref{eq:sticky-chain}.

\begin{lemma}\label{lem:stickymixingtime}
    Suppose $|\Xspace| \geq 2$ and $p \in (0, 1/2]$. For any sticky Markov chain as defined in Eq.~\eqref{eq:sticky-chain}, we have
    \begin{align}
        \frac{1}{2p} \leq \Tmix \leq \frac{2}{p}.
    \end{align}
\end{lemma}
\begin{proof}
    We proceed by exactly calculating the total variation distance $\max_{x \in \Xspace} \| e_x^\top \Pmat^t - \pi^\top \|_{\mathsf{TV}}$.
    For any starting state $x \in \Xspace$ we would reach the stationary distribution $\pi$ in a number of steps that is a geometric random variable, i.e. $\tau = \text{Geom}(p)$. This means that $\Pr\{\tau \geq t\} = ( 1 - p)^t$, directly implying that
    \begin{align*}
        \max_{x \in \Xspace} \|e_x^\top \Pmat^t - \pi^\top \|_{\mathsf{TV}} &= \max_{x \in \Xspace} \frac{1}{2} \|(1 - p)^t \cdot (e_x - \pi)\|_{1} \\
        &= \frac{(1 - p)^t}{2} \cdot \max_{x \in \Xspace} \|e_x - \pi\|_{1}.
    \end{align*}
 Since $(1 - p)^t \cdot \max_{x \in \Xspace} \|e_x - \pi\|_{\mathsf{TV}}$ is monotonically decreasing in $t$, we have
    \begin{align*}
        \Tmix = \tmix(1/4) = \frac{\log (2 \cdot \max_{x \in \Xspace} \|e_x - \pi\|_{1})}{\log(1/(1-p))}.
    \end{align*}
    But 
    \[
    \|e_x - \pi\|_{1} = (1 - \pi_x) + \sum_{y \in \Xspace \setminus \{x\} } \pi_y = 2 - 2\pi_x,
    \]
    and if $|\Xspace| \geq 2$, then we have
    $1 \leq \max_{x \in \Xspace} \|e_x - \pi\|_{1} \leq 2$.
    Furthermore, if $p \in (0, 1/2]$, then $p \leq \log(1/(1-p)) \leq p \log 4$. Putting together the pieces, we obtain the sandwich bound
    \[
    \frac{1}{2p} \leq \Tmix \leq \frac{2}{p},
    \]
    as desired.
\end{proof}

For any set $P \subseteq [n]$, recall that $\bm{X}_P = \{X_k\}_{k \in P}$ denotes the set of random variables in $X^n$ restricted to the index set $P$. The following lemma is a deterministic statement regarding indicator random variables. 

\begin{lemma} \label{lem:diff-ind}

Consider the sequence $X^n$ and any random variable $Y$ defined on the space $\Xspace$. Let $P \subseteq Q \subseteq [n]$ denote two index sets, and let $R := Q \setminus P$. We have
\[
\ind{Y \notin \bm{X}_P} - \ind{Y \notin \bm{X}_Q} = \ind{Y \in \bm{X}_R} \cdot \ind{Y \notin \bm{X}_P}. 
\]
\end{lemma}
\begin{proof}
Since $P \subseteq Q$, we have that $\bm{X}_P \subseteq \bm{X}_Q$. 
Consequently, if $Y \notin \bm{X}_{Q}$, then $Y$ cannot be included in the subset $\bm{X}_P$. Therefore, $\ind{Y \notin \bm{X}_P} - \ind{Y \notin \bm{X}_Q} = 1$ if and only if $Y \notin \bm{X}_P$ \emph{and} $Y \in \bm{X}_Q$. Since $R = Q \setminus P$, this is equivalent to saying that $Y \in \bm{X}_R$ and $Y \notin \bm{X}_P$.
Thus, we have shown that
\begin{align*}
\ind{Y \notin \bm{X}_P} - \ind{Y \notin \bm{X}_Q} &= \ind{Y \in \bm{X}_R} \cdot \ind{Y \notin \bm{X}_P},  
\end{align*}
as claimed.
\end{proof}

We also define an extension of Lemma~\ref{lem:diff-ind} to the slightly more complicated indicator random variables involving the count of an element in index sets.
\begin{lemma}\label{lem:diff-ind-count}
Consider the sequence $X^n$ and any random variable $Y$ defined on the space $\Xspace$. Let $P \subseteq Q \subseteq [n]$ denote two index sets, and let $R := Q \setminus P$. Then, for any $\zeta \geq 0$, we have
\[
\ind{N_Y(\bm{X}_P) \leq \zeta} - \ind{N_Y(\bm{X}_Q) \leq \zeta} \leq \ind{Y \in \bm{X}_R} \cdot \ind{N_Y(\bm{X}_P) \leq \zeta}.
\]
\end{lemma}
\begin{proof}
    Since $P \subseteq Q$, we have that $\bm{X}_P \subseteq \bm{X}_Q$. Then, if $Y$ occurs less than $\zeta$ times in $\bm{X}_Q$, i.e. $N_Y(\bm{X}_Q) \leq \zeta$, then $Y$ must occur less than $\zeta$ times in the subset $\bm{X}_P$, i.e. $N_Y(\bm{X}_P) \leq \zeta$.
    Consequently, $\ind{N_Y(\bm{X}_P) \leq \zeta} - \ind{N_Y(\bm{X}_Q) \leq \zeta} = 1$ if and only if the number of occurrences of $Y$ in $\bm{X}_P$ is less than or equal to $\zeta$, i.e. $N_Y(\bm{X}_P) \leq \zeta$; but the number of occurrences of $Y$ in $\bm{X}_Q$ is greater than $\zeta$, i.e. $N_Y(\bm{X}_Q) > \zeta$.
    Further, we have $N_Y(\bm{X}_P) \leq \zeta$ and $N_Y(\bm{X}_Q) > \zeta$ \emph{only if } $N_Y(\bm{X}_P) \leq \zeta$ \emph{and} there exists at least one occurrence of $Y$ in $\bm{X}_R$, i.e. $N_Y(\bm{X}_R) \geq 1$.
    This gives us
    \begin{align*}
        \ind{N_Y(\bm{X}_P) \leq \zeta} - \ind{N_Y(\bm{X}_Q) \leq \zeta} \leq \ind{Y \in \bm{X}_R} \cdot \ind{N_Y(\bm{X}_P) \leq \zeta}.
    \end{align*}
\end{proof}

\subsection{Lemmas on surrogate processes}

We next present two important consequences of mixing. 
In all the lemmas below, let $(X_1, \ldots, X_n)$ denote a Markov chain with unique stationary distribution $\pi$ and $X_1 \sim \pi$. Let $\tmix(\epsilon)$ denote its mixing time in the sense of Eq.~\eqref{eq:tmix}, with $\epsilon \in (0, 1/2]$.

\begin{lemma} \label{lem:Markov-bridge}
Fix a positive scalar $\epsilon \leq 1/2$, and let $\tau \geq \tmix(\epsilon)$ be an integer.
For each $i \in [n]$, define the stochastic processes
\begin{align}
Z_i &= (X_1, X_2, \ldots, X_{i - \tau}, X_i, X_{i + \tau}, X_{i + \tau + 1}, \ldots, X_n), \\
Z'_i &= (X_1, X_2, \ldots, X_{i - \tau}, X'_i, X_{i + \tau}, X_{i + \tau + 1}, \ldots, X_n),
\end{align}
where $X'_i \sim \pi$ is drawn independently of everything else. Then $\dtv(Z_i, Z'_i) \leq 4\epsilon$.

Consequently, for any function $f: \Xspace^{n - (2\tau-2)} \to [0, 1]$, we have
\[
|\EE [f(Z_i) - f(Z'_i)]| \leq 4 \epsilon.
\]
\end{lemma}
\begin{proof}
Let $A_i = (X_{i - \tau}, X_i, X_{i + \tau})$ and $A'_i = (X_{i - \tau}, X'_i, X_{i + \tau})$. By the Markov property, we have
$\dtv(Z_i,Z'_i) = \dtv(A_i,A'_i)$. 
We now define the notation $p^{(t)}(y|x) = \Pr\{X_{i + t} = y | X_i = x\}$ for any $t \geq 1$ and any $x, y \in \Xspace$.
Owing to the time invariant nature of the process, the distributions of these triples can be written explicitly as
\begin{align*}
&\Pr\{A_i =  (x, y, z) \} = \Pr\{(X_{i - \tau}, X_i, X_{i + \tau}) =  (x, y, z) \} = \pi_x \cdot p^{(\tau)} (y|x) \cdot p^{(\tau)}(z|y) \text{ and } \\
&\Pr\{A'_i =  (x, y, z) \} = \Pr\{(X_{i - \tau}, X'_i, X_{i + \tau}) =  (x, y, z) \} = \pi_x \cdot \pi_y \cdot p^{(2\tau)}(z|x).
\end{align*}
For each tuple of indices $(x, y, z) \in \Xspace \times \Xspace \times \Xspace$ and any choice $t \geq 1$, define
\begin{align*}
\delta_{x, y}^{(t)} := \pi_y - p^{(t)}(y|x) \text{ and } \\
\overline{\delta}^{(t)}_{x, y, z} := p^{(2t)}(z|x) - p^{(t)}(z|y).
\end{align*}
Below, we use the shorthand $\delta_{x,y} := \delta_{x,y}^{(\tau)}$ and $\overline{\delta}_{x,y,z} := \delta_{x,y,z}^{(\tau)}$ for convenience.
Owing to our total variation mixing assumption~\eqref{eq:tmix} and the choice $\tau \geq \tmix(\epsilon)$, we have that the $\ell_1$-norm of each of these errors is bounded for any $t \geq \tau$ as:
\begin{subequations}\label{eq:l1sizeoferrors}
\begin{align}
\max_{x \in \Xspace} \sum_{y \in \Xspace} | \delta_{x,y}^{(t)} | &\leq 2\epsilon \text{ and } \label{eq:l1sizeoferrors-1} \\
\max_{x, y \in \Xspace} \sum_{z \in \Xspace} | \overline{\delta}^{(t)}_{x, y, z} | &\leq 2\epsilon + 2\epsilon^2.
\end{align}
\end{subequations}
With this shorthand notation, we may define
\begin{align*}
&\Pr\{A_i =  (x,y,z) \} = \pi_x \cdot p^{(\tau)}(y|x) \cdot p^{(\tau)}(z|y) \text{ and } \\
&\Pr\{A'_i =  (x,y,z) \} = \pi_x \cdot (p^{(\tau)}(y|x) + \delta_{x,y}) \cdot (p^{(\tau)}(z|y) + \overline{\delta}_{x, y, z})
\end{align*}
and we can write the desired total variation explicitly as
\begin{align}\label{eq:tvonebridge}
&\dtv(A_i, A'_i) \notag \\
&= \frac{1}{2} \sum_{x, y, z \in \Xspace} | \pi_x \cdot p^{(\tau)}(y|x) \cdot p^{(\tau)}(z|y) - \pi_x \cdot (p^{(\tau)}(y|x) + \delta_{x,y}) \cdot (p^{(\tau)}(z|y) + \overline{\delta}_{x,y,z}) | \notag \\
&= \frac{1}{2} \sum_{x, y, z \in \Xspace} | \pi_x \cdot p^{(\tau)}(y|x) \cdot \overline{\delta}_{x, y, z} + \pi_x \cdot p^{(\tau)}(z|y) \cdot \delta_{x, y} + \pi_x \cdot \delta_{x, y} \cdot \overline{\delta}_{x,y, z} |
 \notag \\
&\leq  \sum_{x,y,z \in \Xspace} \frac{1}{2} | \pi_x \cdot p^{(\tau)}(y|x) \cdot \overline{\delta}_{x,y,z}| + \sum_{x,y,z \in \Xspace} \frac{1}{2} |\pi_x \cdot p^{(\tau)}(z|y) \cdot \delta_{x, y}| + \sum_{x,y,z \in \Xspace} \frac{1}{2} |\pi_x \cdot \delta_{x,y} \cdot \overline{\delta}_{x,y,z}| \notag \\
&\overset{\1}{\leq} (\epsilon + \epsilon^2) + \epsilon + (\epsilon + \epsilon^2) = 3\epsilon + 2 \epsilon^2 \leq 4\epsilon,
 \end{align}
 where we claim that the bound in step $\1$ holds term by term and the last inequality uses the fact that $\epsilon \leq 1/2$. It remains to prove step $\1$. The first term can be bounded as
 \begin{align*}
 \sum_{x,y,z \in \Xspace} | \pi_x \cdot p^{(\tau)}(y|x) \cdot \overline{\delta}_{x,y,z}| = \sum_{x,y \in \Xspace} \pi_x \cdot p^{(\tau)}(y|x) \sum_{z \in \Xspace}  | \overline{\delta}_{x, y, z}| &\leq \sum_{x,y \in \Xspace} \pi_x \cdot p^{(\tau)}(y|x) \cdot (2 \epsilon + 2 \epsilon^2) \\
 &= 2 \epsilon + 2\epsilon^2.
 \end{align*}
 The remaining terms can be bounded using similar logic, so we omit the steps for brevity.

 The consequence of the TV bound for expectations of bounded functions follows by the definition of total variation distance.
\end{proof}

\begin{lemma} \label{lem:two-bridges}
Fix a positive scalar $\epsilon \leq 1/2$, and let $\tau \geq \tmix(\epsilon)$ be an integer.
For each $i_1 < i_2 \in [n]$ with $i_2 - i_1 \geq 2\tau$, define the stochastic sub-processes
\begin{align}
Z_{i_1, i_2} &= (X_1, X_2, \ldots, X_{i_1 - \tau}, X_{i_1}, X_{i_1 + \tau}, \ldots, X_{i_2 - \tau}, X_{i_2}, X_{i_2 + \tau}, \ldots, X_n), \\
Z'_{i_1, i_2} &= (X_1, X_2, \ldots, X_{i_1 - \tau}, X'_{i_1}, X_{i_1 + \tau}, \ldots, X_{i_2 - \tau}, X'_{i_2}, X_{i_2 + \tau}, \ldots, X_n),
\end{align}
where $X'_{i_1}, X'_{i_2} \sim \pi$ are drawn independently of each other and of everything else. 
Then we have \mbox{$\dtv(Z_{i_1, i_2}, Z'_{i_1, i_2}) \leq 8\epsilon$.}

Consequently, for any function $f$ with range $[0, 1]$, we have
\[
|\EE [f(Z_{i, j}) - f(Z'_{i, j})]| \leq 8 \epsilon.
\]
\end{lemma}
\begin{proof}
We prove the bound on total variation, noting that the consequence for bounded functions follows as a corollary. 

As in the proof of Lemma~\ref{lem:Markov-bridge}, define the sub-processes 
\begin{align*}
    A_{i_1,i_2} &= (X_{i_1-\tau},X_{i_1},X_{i_1 + \tau},X_{i_2 - \tau},X_{i_2},X_{i_2 + \tau}), \\
    A'_{i_1,i_2} &= (X_{i_1-\tau},X'_{i_1},X_{i_1 + \tau },X_{i_2 - \tau},X'_{i_2},X_{i_2 + \tau }).
\end{align*}
(Note that in the special case where $i_2 - i_1 = 2\T$, we have $i_1 + \T = i_2 - \T$ and so the above definition remains valid --- the sub-process in this case contains a duplicated random variable $X_{i_1 + \T} = X_{i_2 - \T}$.)

We also define an intermediate sub-process $\widetilde{A}_{i_1,i_2} = (X_{i_1-\tau},X'_{i_1},X_{i_1 + \tau},X_{i_2 - \tau},X_{i_2},X_{i_2 + \tau})$ for convenience. Then, simplifying the expression for total variation over the entire Markov chain and noting that the stochastic processes $Z_{i_1,i_2},Z'_{i_1,i_2}$ follow identical transition laws over the indices $\{1,\ldots,i_1-\tau\} \cup \{i_1 + \tau + 1,\ldots,i_2 - \tau\} \cup \{i_2 + \tau + 1,\ldots,n\}$, we have
\begin{align*}
    \dtv(Z_{i_1,i_2},Z'_{i_1,i_2}) = \dtv(A_{i_1,i_2},A'_{i_1,i_2}) \overset{\1}{\leq} \dtv(A_{i_1,i_2},\widetilde{A}_{i_1,i_2}) + \dtv(\widetilde{A}_{i_1,i_2}, A'_{i_1,i_2}),
\end{align*}
where step $\1$ follows by the triangle inequality.
We proceed to upper bound each of the terms $\dtv(A_{i_1,i_2},\widetilde{A}_{i_1,i_2})$ and $\dtv(\widetilde{A}_{i_1,i_2}, A'_{i_1,i_2})$ by $4\epsilon$ each, using a similar argument to the proof of Lemma~\ref{lem:Markov-bridge}.
We denote $\rho_{y_2,w_1} = \Pr[X_{i_2 - \tau} = y_2 | X_{i_1 + \tau} = w_1]$ as shorthand, noting that for each $w_1 \in \Xspace$, we have $\sum_{y_2 \in \Xspace} \rho_{y_2,w_1} = 1$.
(Note that in the special case where $i_2 - i_1 = 2\T$, we have $X_{i_1 + \T} = X_{i_2 - \T}$ and this conditional distribution takes on the special form $\rho_{y_2,w_1} = \ind{y_2 = w_1}$.)

We begin with the first term $\dtv(A_{i_1,i_2},\widetilde{A}_{i_1,i_2})$ and characterize the distributions of $A_{i_1,i_2},\widetilde{A}_{i_1,i_2}$ as below:
\begin{align*}
    &\Pr\{A_{i_1,i_2} =  (y_1, z_1, w_1,y_2,z_2,w_2) \} \\
    &\qquad \qquad \qquad \qquad \qquad \qquad = \pi_{y_1} \cdot p^{(\tau)}(z_1|y_1) \cdot p^{(\tau)}(w_1|z_1) \cdot \rho_{y_2,w_1} \cdot p^{(\tau)}(z_2|y_2) \cdot p^{(\tau)}(w_2|z_2) \\
    &\Pr\{\widetilde{A}_{i_1,i_2} =  (y_1, z_1, w_1,y_2,z_2,w_2) \} \\
    &\qquad \qquad \qquad \qquad \qquad \qquad = \pi_{y_1} \cdot \pi_{z_1} \cdot p^{(2\tau)}(w_1|y_1) \cdot \rho_{y_2,w_1} \cdot p^{(\tau)}(z_2|y_2) \cdot p^{(\tau)}(w_2|z_2).
\end{align*}
Recalling the shorthand notation $\delta_{x,y}, \deltabar_{x,y,z}$ defined in the proof of Lemma~\ref{lem:Markov-bridge}, we can write the above as
\begin{align*}
   &\Pr\{A_{i_1,i_2} =  (y_1, z_1, w_1,y_2,z_2,w_2) \} \\
   &\qquad= \pi_{y_1} \cdot p^{(\tau)}(z_1|y_1) \cdot p^{(\tau)}(w_1|z_1) \cdot \rho_{y_2,w_1} \cdot p^{(\tau)}(z_2|y_2) \cdot p^{(\tau)}(w_2|z_2) \\
   &\Pr\{\widetilde{A}_{i_1,i_2} =  (y_1, z_1, w_1,y_2,z_2,w_2) \} \\
   &\qquad= \pi_{y_1} \cdot (p^{(\tau)}(z_1|y_1) + \delta_{y_1,z_1}) \cdot (p^{(\tau)}(w_1|z_1) + \deltabar_{y_1,z_1,w_1}) \cdot \rho_{y_2,w_1} \cdot p^{(\tau)}(z_2|y_2) \cdot p^{(\tau)}(w_2|z_2).
\end{align*}
Next, we note that $\rho_{y_2,w_1} \cdot p^{(\tau)}(z_2|y_2) \cdot p^{(\tau)}(w_2|z_2) = \Pr[X_{i_2 - \tau} = y_2,X_{i_2} = z_2,X_{i_2 + \tau} = w_2 | X_{i_1 + \tau} = w_1]$ which is a conditional probability distribution that is identical for the stochastic processes $\widetilde{A}_{i_1,i_2}$ and $A_{i_1,i_2}$.
Therefore, we have $\sum_{y_2,z_2,w_2 \in \Xspace} \rho_{y_2,w_1} \cdot p^{(\tau)}(z_2|y_2) \cdot p^{(\tau)}(w_2|z_2) = 1$ for any value of $w_1 \in \Xspace$.
This yields
\begin{align*}
    &\dtv(A_{i_1,i_2},\widetilde{A}_{i_1,i_2}) \\
    &= \frac{1}{2} \sum_{\substack{y_1,z_1,w_1 \in \Xspace \\ y_2,z_2,w_2 \in \Xspace}} \Big|\pi_{y_1} \cdot p^{(\tau)}(z_1|y_1) \cdot p^{(\tau)}(w_1|z_1) \cdot \rho_{y_2,w_1} \cdot p^{(\tau)}(z_2|y_2) \cdot p^{(\tau)}(w_2|z_2) \\
    &\qquad -  \pi_{y_1} \cdot (p^{(\tau)}(z_1|y_1) + \delta_{y_1,z_1}) \cdot (p^{(\tau)}(w_1|z_1) + \deltabar_{y_1,z_1,w_1}) \cdot \rho_{y_2,w_1} \cdot p^{(\tau)}(z_2|y_2) \cdot p^{(\tau)}(w_2|z_2) \Big| \\
    &= \sum_{y_1,z_1,w_1 \in \Xspace} \frac{1}{2} \Big| \pi_{y_1} p^{(\tau)}(z_1|y_1) p^{(\tau)}(w_1|z_1) - \pi_{y_1} (p^{(\tau)}(z_1|y_1) + \delta_{y_1,z_1}) (p^{(\tau)}(w_1|z_1) + \deltabar_{y_1,z_1,w_1})  \Big|.
\end{align*}
We have thus arrived at an expression for $\dtv(A_{i_1,i_2},\widetilde{A}_{i_1,i_2})$ that is identical to Equation~\eqref{eq:tvonebridge}, which is upper bounded by $4\epsilon$.
Therefore, we have $\dtv(A_{i_1,i_2},\widetilde{A}_{i_1,i_2}) \leq 4\epsilon$.

We now use a similar technique to bound the other term $\dtv(\widetilde{A}_{i_1,i_2}, A'_{i_1,i_2})$.
In particular, we have
\begin{align*}
      &\Pr\{\widetilde{A}_{i_1,i_2} =  (y_1, z_1, w_1,y_2,z_2,w_2) \} \\
      &= \pi_{y_1} \cdot \pi_{z_1} \cdot p^{(2\tau)}(w_1|y_1) \cdot \rho_{y_2,w_1} \cdot p^{(\tau)}(z_2|y_2) \cdot p^{(\tau)}(w_2|z_2) \\
      &\Pr\{A'_{i_1,i_2} =  (y_1, z_1, w_1,y_2,z_2,w_2) \} \\
      &= \pi_{y_1} \cdot \pi_{z_1} \cdot p^{(2\tau)}(w_1|y_1) \cdot \rho_{y_2,w_1} \cdot (p^{(\tau)}(z_2|y_2) + \delta_{y_2,z_2}) \cdot (p^{(\tau)}(w_2|z_2) + \deltabar_{y_2,z_2,w_2}).
\end{align*}
This time, we note that $\pi_{y_1} \cdot \pi_{z_1} \cdot p^{(2\tau)}(w_1|y_1) \cdot \rho_{y_2,w_1} = \Pr[X_{i_1 - \tau} = y_1,X'_{i_1} = z_1,X_{i_1 + \tau} = w_1, X_{i_2 - \tau} = y_2]$ and therefore $\sum_{y_1,z_1,w_1 \in \Xspace} \pi_{y_1} \cdot \pi_{z_1} \cdot p^{(2\tau)}(w_1|y_1) \cdot \rho_{y_2,w_1} = \Pr[X_{i_2 - \T} = y_2] = \pi_{y_2}$.
Using a similar series of steps to the preceding calculation, we obtain
\begin{align*}
    &\dtv(\widetilde{A}_{i_1,i_2},A'_{i_1,i_2}) \\
    &= \sum_{y_2,z_2,w_2 \in \Xspace} \frac{1}{2} 
 \Big| \pi_{y_2} p^{(\tau)}(z_2|y_2) p^{(\tau)}(w_2|z_2) - \pi_{y_2}(p^{(\tau)}(z_2|y_2) + \delta_{y_2,z_2})(p^{(\tau)}(w_2|z_2) + \deltabar_{y_2,z_2,w_2}) \Big| \\
 &\leq 4\epsilon
\end{align*}
by an identical argument to Eq.~\eqref{eq:tvonebridge}.
Putting these together yields $\dtv(Z_{i_1,i_2}, Z'_{i_1,i_2}) \leq 8 \epsilon$.
\end{proof}

\section{Intuition for data-dependent tuning of window size $\tau$}\label{sec:tuning-justification}

In this section, we provide some simple intuition to justify the data-dependent tuning procedure for the window size $\T$ that we described in Section~\ref{sec:expts-tuning}.
Assuming that\footnote{If $n \lesssim \Tmix$, it is impossible to obtain consistent estimation anyway, at least in a minimax sense.} $n \gg \Tmix$, we have that $Z^{(1)}$ and $Z^{(2)}$ are near-independent since they are significantly separated within the sequence. Thus, conditioned on $Z^{(1)}$, the sequence $Z^{(2)}$ should be thought of as an independent Markov chain started at the stationary distribution $\pi$. Consequently, the random variable $\widetilde{M}(Z^{(1)})$ ought to be close to the estimand $M_{\pi}(Z^{(1)})$, and this can be formalized via a bounded differences inequality for mixing Markov chains~\citep{paulin2015concentration}. Indeed, if independence between $Z^{(1)}$ and $Z^{(2)}$ held exactly, then it is straightforward to show that with high probability over the randomness in $Z^{(2)}$, we have \begin{align} \label{eq:bdd-diff-MC}
| \widetilde{M}(Z^{(1)}) - M_{\pi}(Z^{(1)}) |^2 \lesssim \frac{\Tmix \log (n / \Tmix)}{n}.
\end{align}
Now Theorem~\ref{thm:main-upper-bound} guarantees that for some $\T_0 \asymp \Tmix \log (n/ \Tmix)$, we must have the inequality\footnote{Note that this step of the argument is heuristic, since Theorem~\ref{thm:main-upper-bound} only gives such a guarantee in expectation.} \\
\mbox{$\left| \Mhat_{\some}(Z^{(1)}; \T_0) - M_{\pi}(Z^{(1)}) \right|^2 \lesssim \frac{\T}{n}$}. Combining this observation with Ineq.~\eqref{eq:bdd-diff-MC} and noting that $\T_0 \asymp \Tmix \log (n/ \Tmix)$, we have
\begin{align*}
\left| \Mhat_{\some}(Z^{(1)}; \T_0) - \widetilde{M}(Z^{(1)}) \right|^2 &\leq 2 \left| \Mhat_{\some}(Z^{(1)}; \T_0) - M_{\pi}(Z^{(1)}) \right|^2 + 2 \left| M_{\pi}(Z^{(1)}) - \widetilde{M}(Z^{(1)}) \right|^2 \\
&\lesssim \frac{\T_0}{n}.
\end{align*}
Thus, we see that Ineq.~\eqref{eq:tuning-T-choice} is a reasonable validation criterion since it is satisfied for some choice of window size at most $\tau_0$, for a suitable choice of constant $C_{\mathsf{tune}}$ on the RHS.
Conversely, if Ineq.~\eqref{eq:tuning-T-choice} holds for some smaller window size $\T = \widehat{\tau} \leq \tau_0$, then combining this with Ineq.~\eqref{eq:bdd-diff-MC} yields
\begin{align*}
\frac{1}{2} \left| \Mhat_{\some}(Z^{(1)}; \widehat{\tau}) - M_{\pi}(Z^{(1)}) \right|^2 &\leq \left| \Mhat_{\some}(Z^{(1)}; \widehat{\tau}) - \widetilde{M}(Z^{(1)}) \right|^2 + \left| M_{\pi}(Z^{(1)}) - \widetilde{M}(Z^{(1)}) \right|^2 \\
&\lesssim \frac{\widehat{\tau}}{n} + \frac{\Tmix \log (n / \Tmix)}{n} \\
&\leq \frac{\tau_0}{n} + \frac{\Tmix \log (n / \Tmix)}{n} \lesssim \frac{\Tmix \log (n / \Tmix)}{n}.
\end{align*}
Putting together the pieces, we see that our validation procedure is reasonable since (a) It is satisfied by the window size $\tau_0$ prescribed by Theorem~\ref{thm:main-upper-bound}, and (b) It always produces a good value of tuning window size $\widehat{\tau}$, in that this choice of window size leads to the optimal rate of estimation of the functional $M_{\pi}(Z^{(1)})$.

It is important to note that the above sketch does not constitute a rigorous argument. In order to make it rigorous, one would have to formally establish Eq.~\eqref{eq:bdd-diff-MC} and also a version of Theorem~\ref{thm:main-upper-bound} that holds with high probability, both of which are interesting directions for future work.
\end{document}